\pdfoutput=1
\documentclass[letterpaper]{article} 
\usepackage{aaai24}  
\usepackage{times}  
\usepackage{helvet}  
\usepackage{courier}  
\usepackage[hyphens]{url}  
\usepackage{graphicx} 
\usepackage{natbib}  
\usepackage{caption} 
\usepackage{amsmath,amsthm,amssymb}
\newtheorem{definition}{Definition}
\newtheorem{lemma}{Lemma}
\newtheorem{theorem}{Theorem}
\newtheorem{proposition}{Proposition}

\newtheorem{corollary}{Corollary}

\usepackage{array}
\usepackage{enumerate}
\usepackage{microtype}      
\usepackage{xcolor}         
\usepackage[linesnumbered,ruled]{algorithm2e}
\usepackage{booktabs}       
\usepackage{tikz}
\usetikzlibrary{arrows.meta, graphs,calc}
\tikzset{
  pico/.style = {
    every node/.style = {
      draw,
      circle,
      semithick,
      inner sep = 0pt,
      minimum width = 0.7ex,
      fill = white
    },
    semithick
  },
    edge/.style = {
    semithick
  },
  arc/.style = {
    ->,
    semithick,
    >={[round,sep]Stealth}
  },
  bidi/.style = {
    <->,
    dashed,
    semithick,
    >={[round,sep]Stealth}
  }
}
\usepackage{etoolbox}
\makeatletter
\patchcmd\algocf@Vline{\vrule}{\vrule \kern-0.4pt}{}{}
\patchcmd\algocf@Vsline{\vrule}{\vrule \kern-0.4pt}{}{}
\makeatother

\definecolor{ba.yellow}{RGB}{252,190,18}
\definecolor{ba.gray}{RGB}{153,153,156}
\definecolor{ba.blue}{RGB}{6,123,164}
\definecolor{ba.red}{RGB}{213,96,98}
\definecolor{ba.orange}{RGB}{233,116,81}
\definecolor{ba.pine}{RGB}{67,154,134}
\definecolor{ba.green}{RGB}{0, 168, 107}
\definecolor{ba.lightgreen}{RGB}{196,247,161}
\definecolor{ba.violet}{RGB}{88, 53, 94}

\newcommand{\bA}{{\bf A}}
\newcommand{\bB}{{\bf B}}
\newcommand{\bC}{{\bf C}}

\newcommand{\bE}{{\bf E}}

\newcommand{\bI}{{\bf I}}

\newcommand{\bN}{{\bf N}}

\newcommand{\bR}{{\bf R}}
\newcommand{\bS}{{\bf S}}

\newcommand{\bV}{{\bf V}}
\newcommand{\bW}{{\bf W}}
\newcommand{\bY}{{\bf Y}}
\newcommand{\bX}{{\bf X}}
\newcommand{\bZ}{{\bf Z}}

\newcommand{\by}{{\bf y}}
\newcommand{\bx}{{\bf x}}
\newcommand{\bz}{{\bf z}}

\def\bZii{\bZ_{(\mathrm{ii})}}

\newcommand{\Pa}{\textit{\textbf{Pa}}} 
\newcommand{\Ch}{\textit{\textbf{Ch}}} 
\newcommand{\Ne}{\textit{\textbf{Ne}}} 

\newcommand{\An}{\textit{\textbf{An}}} 
\newcommand{\De}{\textit{\textbf{De}}}
\makeatletter
\newcommand*{\indep}{%
  \mathbin{%
    \mathpalette{\@indep}{}%
  }%
}
\newcommand*{\nindep}{%
  \mathbin{
    \mathpalette{\@indep}{/}%
  }%
}
\newcommand*{\@indep}[2]{%
  \sbox0{$#1\perp\m@th$}
  \sbox2{$#1=$}
  \sbox4{$#1\vcenter{}$}
  \rlap{\copy0}
  \dimen@=\dimexpr\ht2-\ht4-.2pt\relax
  \kern\dimen@
  \ifx\\#2\\%
  \else
    \hbox to \wd2{\hss$#1#2\m@th$\hss}%
    \kern-\wd2 %
  \fi
  \kern\dimen@
  \copy0 
}
\makeatother

\title{Linear-Time Algorithms for Front-Door Adjustment in Causal Graphs\thanks{Extended
  version of paper accepted to the Proceedings of the 38th AAAI
Conference on Artificial Intelligence (AAAI-24).}}
\author{
    Marcel Wienöbst, Benito van der Zander, Maciej Li\'{s}kiewicz \\
}
\affiliations{
    Institute of Theoretical Computer Science, University of L\"{u}beck, Germany \\
    \{m.wienoebst,b.vanderzander,maciej.liskiewicz\}@uni-luebeck.de
}

\begin{document}

\maketitle\vspace*{-9.51mm}
\thispagestyle{plain}

\begin{abstract}
Causal effect estimation from observational data is a fundamental task
in empirical sciences. It becomes particularly challenging when
unobserved confounders are involved in a system. This paper focuses on
front-door adjustment -- a classic technique which, using observed
mediators allows to identify causal effects even in the presence of
unobserved confounding. While the statistical properties of the
front-door estimation are quite well understood, its algorithmic
aspects remained unexplored for a long time. In 2022, Jeong, Tian, and Bareinboim 
presented the first polynomial-time
algorithm for finding sets satisfying the front-door criterion in a
given directed acyclic graph (DAG), with an $O(n^3(n+m))$ run time,
where $n$ denotes the number of variables and $m$ the number of edges
of the causal graph. In our work, we give the first linear-time, i.e.,
$O(n+m)$, algorithm for this task, which thus reaches the
asymptotically optimal time complexity. This result implies an
$O(n(n+m))$ delay enumeration algorithm of all front-door adjustment
sets, again improving previous work by a factor of
$n^3$. Moreover, we provide the first linear-time algorithm for
finding a \emph{minimal} front-door adjustment set. We offer implementations of our algorithms in multiple programming languages to facilitate practical usage and empirically validate their feasibility, even for large graphs.
\end{abstract}

\def\mmid{ | }

\section{Introduction}
Discovering and understanding causal relationships and distinguishing 
them from purely statistical associations
is a fundamental objective of empirical 
sciences. For example, recognizing the causes of diseases and other health 
problems is a central task in medical research enabling novel disease
prevention and treatment strategies. 
One possible approach for establishing causal relationships and analyzing causal effects is through 
Randomized Controlled Trials \citep{fisher1936design}, which are
considered the \emph{gold standard of experimentation}.  In practice, however,  experimentation is not 
always possible due to costs, technical
feasibility, or ethical constraints
-- e.g., participants of medical studies should not be assigned to smoke over
extended periods of time to ascertain its harmfulness.

The goal of \emph{causal inference} is to determine cause-effect relationships 
by combining observed and interventional data with existing knowledge. In this paper, 
we focus on the problem of deciding when causal effects can be identified from 
a graphical model \emph{and} observed data and, if possible, how to estimate the 
strength of the effect.  The model is typically represented as a directed acyclic graph 
(DAG), whose edges encode direct causal influences between the random variables 
of interest. To analyze the causal effects in such models,
\citet{pearl1995causal,Pearl2009} introduced
the do-operator which performs a hypothetical 
intervention forcing exposure (treatment) variables $\bX$ to take some values $\bx$. 
This allows to regard the (total) \emph{causal effect} of $\bX$ on outcome variables $\bY$, 
denoted as $P(\by \mmid  \textit{do}(\bx))$, as the probability distribution 
of variables $\bY$ after the intervention.\footnote{Following
  convention, for a random variable $X$, 
  we use $P(x)$  as a shorthand for $P(X=x)$. 
By bold capital letters $\bX,\bY$, etc., we denote sets of variables, and the corresponding 
sets of values are denoted by bold lowercase letters $\bx,\by$, etc.}
The fundamental task in causal inference is to decide whether 
$P(\by \mmid  \textit{do}(\bx))$ can be expressed using only \emph{standard} 
(i.e., do-operator free) probabilities and it becomes challenging when unobserved 
confounders (variables affecting both the treatment and outcome)
are involved in a system.
A variable is considered unobserved if it cannot be measured by the researcher. 
Fig.~\ref{fig:FD:examples} shows DAGs of some models with unobserved 
confounders. 

\begin{figure}
  \begin{center}
    \begin{tikzpicture}
      \node (x) at (0,0) {$X$};
      \node (z) at (1,0) {$Z$};
      \node (y) at (2,0) {$Y$};
      \node (u) at (1,0.7) {$U$};
      \node (l) at (-0.2,0.6) {(i)};
      \graph[use existing nodes, edges = {arc}] {
        x -- z;
        z -- y;
        u -- [dashed, bend right=25] (x);
        u -- [dashed, bend left=25] (y);
      };
    \end{tikzpicture} \hspace*{0.5cm}
    \begin{tikzpicture}
     \node (x) at (0,0) {$X$};
     \node (a) at (1,0) {$A$};
     \node (b) at (2,0.5) {$B$};
     \node (c) at (2,0) {$C$};
     \node (d) at (3,0) {$D$};
     \node (y) at (4,0) {$Y$};
     \node (u) at (1.7,0.8) {$U$};
     \node (l) at (0.0,0.6) {(ii)};
     \graph[use existing nodes, edges = {arc}] {
       x -- a -- b -- d -- y;
       a -- c -- d;
       u -- [dashed, bend right=11] (x);
       u -- [dashed, bend left=20] (y);
     };
    \end{tikzpicture}
    
    \vspace*{0.5cm}

    \begin{tikzpicture}
      \node (x) at (0,0) {$X$};
      \node (a) at (1,0) {$A$};
      \node (y) at (2,0) {$Y$};
      \node (u) at (1,0.7) {$U$};
      \node (d) at (4.0,0.0) {$D$};
      \node (c) at (3,0) {$C$};
      \node (b) at (3,0.7) {$B$};
      \node (l) at (-0.3,0.6) {(iii)};
      \graph[use existing nodes, edges = {arc}] {
        x -- a;
        a -- y;
        b -- y;
        c -- y;
        d -- b;
        d -- c;
        u -- [dashed, bend right=25] (x);
        u -- [dashed, bend left=25] (y);
      };
    \end{tikzpicture} 
  \end{center}
  \caption{Causal graphs, where $X$ is the treatment,
  $Y$ the outcome, and 
  $U$ represents an unobserved confounder.
  Graph (i) is a canonical example, with the (unique) set
  $\{Z\}$ satisfying the front-door (FD) criterion relative to $(X,Y)$.
  For graph (ii),  there exist 13 FD sets;
  Both the algorithm of \citeauthor{jeong2022finding} 
  as well as our basic Algorithm~\ref{alg:finding}, output $\bZ=\{A,B,C,D\}$ of maximum size.  
  In contrast, Algorithm~\ref{alg:finding-minimal} computes minimal
  FD set $\{D\}$
  of size 1.
   Graph (iii) illustrates the non-monotonicity of the FD 
  criterion: while  both $\{A,B,C\}$ and $\{A\}$ are FD sets,
  neither  $\{A,B\}$ nor  $\{A,C\}$ nor  $\{B,C\}$ satisfy the FD criterion.
}
\label{fig:FD:examples}
\end{figure}
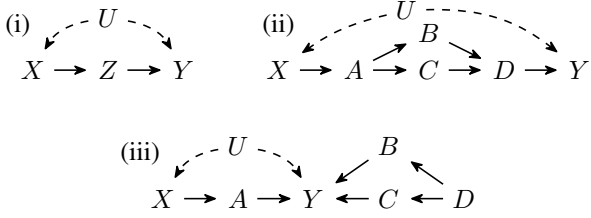

It is well known that the IDC algorithm by \citet{ShpitserIDCAlgorithm}, based on 
the prominent do-calculus \citep{pearl1995causal},
allows solving the identifiability problem in a sound and complete way \citep{huang2006pearl,shpitser2006identification}.
As such, researchers could potentially apply IDC to decide identifiability. 
However, two drawbacks affect the widespread use of the algorithm: Firstly 
it runs in polynomial time of high degree which precludes computations for 
graphs involving a reasonable amount of variables.
Secondly, the IDC algorithm computes complex expressions, 
even in case of small DAGs for which a simple formula exists (for
details, see, e.g., 
\citep{van2019separators}). 
Hence, in practice, total causal effects are estimated using other methods.

One of the most popular approaches  is to utilize \emph{covariate adjustment} of the form 
$P(\by \mmid  \textit{do}(\bx)) = \sum_\bz P(\by \mmid \bx, \bz) P(\bz)$, which 
is valid if $\bZ$ satisfies the famous back-door (BD) criterion \citep{pearl1995causal}.
Apart from convenient statistical properties, the BD based methods
rely on efficient, 
sound, and complete algorithms for covariate adjustments in DAGs. 
In particular,
utilizing the generalized BD criterion by \citet{ShpitserVR2010},  
the algorithmic framework provided by \citet{van2014constructing}
allows to find an adjustment set in linear-time $O(n+m)$ 
and to enumerate all covariate adjustment sets with \emph{delay}
$O(n(n+m))$, i.e., at most $O(n(n+m))$ time passes between successive
outputs.\footnote{By $n$, we denote the number of variables/vertices, by $m$
the number of edges in the causal graph.}

However, BD based approaches are unable to identify the causal effect
in many 
cases involving unobserved 
confounders, which are commonplace in practice.
For example,  none of the instances  in Fig.~\ref{fig:FD:examples}
can be identified via covariate adjustment, 
although   $P(y \mmid \textit{do}(x))$ can be expressed 
by the formula~\eqref{eq:fd:adjustment} below.
This  illustrates the use of  another classic technique,  
known as  front-door (FD) adjustment  \citep{pearl1995causal}, 
which is the main focus of this paper.
The advantage of this approach, as seen in the example, 
is that it leverages observed mediators to identify 
causal effects even in the presence of unobserved confounding.
In the general case, if a set of variables $\bZ$ satisfies the FD 
criterion\footnote{For a definition of the front-door criterion, see
Sec.~\ref{sec:prel}. We term sets satisfying this criterion \emph{FD
sets}.}
relative to $(\bX, \bY)$ in a DAG $G$, the variables $\bZ$ are observed and 
$P(\bx,\bz)>0$, then the 
effect of $\bX$ on $\bY$ is identifiable and is given by the formula
\begin{equation}
 	\label{eq:fd:adjustment}
         P(\by\mmid \textit{do}(\bx)) = \sum_{\bz} P(\bz \mmid \bx)
         \sum_{\bx'} P(\by \mmid \bx',\bz) \ P(\bx'),
\end{equation}
respectively $P(\by\mmid \textit{do}(\bx)) = P(\by)$ in case $\bZ = \emptyset$.
  
FD adjustment is an effective alternative to standard covariate
adjustment~\citep{glynn2018front} 
and is met with increasing applications in real-world datasets~\citep{bellemare2019paper,gupta2021estimating,chinco2016misinformed,cohen2014friends}.
Recent works~\citep{kuroki2000selection,glynn2018front,gupta2021estimating}
have improved the understanding of the statistical properties of FD estimation
and provided robust generalizations of this approach~\citep{hunermund2019causal,fulcher2020robust}.

However, despite these advantages, the algorithmic and
complexity-theoretical 
aspects of FD adjustment remained unexplored for a long time.
Very recently,  \citet{jeong2022finding} have provided the first
polynomial-time algorithm for finding an FD 
adjustment set with an $O(n^3 (n+m))$ run time. This amounts to $O(n^5)$ for dense graphs, 
which does not scale well even for a moderate number of variables. 
The authors also gave an algorithm for enumerating all FD sets, 
which has delay $O(n^4 (n+m))$.
Yet, it remained open, whether these tasks can be solved more efficiently. 

Additionally, the $O(n^3(n+m))$ algorithm by~\citet{jeong2022finding} always returns the
maximum size FD set $\bZ$, which is often unpractical as it hinders the
estimation of FD adjustment formula~\eqref{eq:fd:adjustment}, which sums
over all possible values $\bz$. 
Example (ii) in Fig.~\ref{fig:FD:examples} illustrates this issue: while $\{A,B,C,D\}$ 
is a valid FD adjustment, it is not \emph{minimal} since its proper subset, e.g.,  $\{A\}$, satisfies 
the FD criterion, as well. In this work, we address this issue of
finding a minimal FD set. A feature, which may make the problem difficult to solve, 
is the non-monotonicity of the FD criterion, illustrated in  Fig.~\ref{fig:FD:examples}.(iii). 

The main contributions of this work are threefold:
\begin{itemize}
\item We present the first linear-time, that is $O(n+m)$, algorithm,
  for finding an FD adjustment set. 
	This run time is asymptotically optimal, as the size of the input is $\Omega(n+m)$. 
\item For enumeration of FD adjustment sets, we provide an $O(n (n+m))$-delay algorithm.
\item We give the first linear-time algorithm for finding a
  \emph{minimal} FD set. 
\end{itemize}
Thus, our results show that, in terms of computational complexity,
the problems of finding and enumerating FD sets are not harder than for
the well-studied covariate adjustment and indeed, our algorithms match
the run time of the methods used in this setting. 
In addition, our work offers implementations of the new algorithms in multiple 
programming languages to facilitate practical 
usage and we empirically validate their feasibility, even for large graphs.

\section{Preliminaries}\label{sec:prel}
A directed graph $G = (\bV,\bE)$ consists of a set of vertices (or variables)
$\bV$ and a set of
directed edges $\bE \subseteq \bV \times \bV$. In case of a directed edge $A
\rightarrow B$, vertex $A$ is called a \emph{parent} of $B$ and $B$ is a
\emph{child} of $A$. In case
there is a \emph{causal} path $A \rightarrow \dots \rightarrow
B$, then $A$ is called an \emph{ancestor} of $B$ and $B$ is a
\emph{descendant} of $A$. Vertices are descendants and ancestors of themselves,
but not parents/children. The sets of
parents, children, ancestors, and descendants of a vertex $V$ are denoted by
$\Pa(V)$, $\Ch(V)$, $\An(V)$, and $\De(V)$, and they generalize to
sets $\bV$ in the natural way.
We consider only acyclic graphs (DAGs), i.e., if $B \in \De(A)$, then there is
no edge $B \rightarrow A$. We denote by $G_{\overline{\bS}}$  the graph
obtained by removing from $G$ all edges $\rightarrow S$ for every $S \in \bS$, and
by $G_{\underline{\bS}}$, the graph obtained  by removing  $\leftarrow S$ for every $S \in \bS$.

The statement $(\bA \indep \bB \mid \bC)_G$ in a DAG $G$ holds for pairwise disjoint sets
of vertices $\bA, \bB,
\bC \subset \bV$ if $\bA$ and $\bB$ are \emph{d-separated} in $G$ given $\bC$ -- that is,
if there is no \emph{open} path from some vertex $A \in \bA$ to a vertex $B \in \bB$
given $\bC$. A \emph{path} is a sequence of adjacent, pairwise different vertices and
 it  is \emph{open} if, for any collider $Y$ on the path (that is $X
\rightarrow Y
\leftarrow Z$), we have $\De(Y) \cap \bC \neq \emptyset$, and, for any non-collider
$Y$, we have $Y \not\in \bC$.
In this work, we sometimes consider \emph{ways} instead of paths. A way may
contain a vertex up to two times 
and is open given $\bC$ if, for any collider $Y$, we have $Y \in
\bC$ and, for any non-collider $Y$, we have $Y \not\in \bC$. In case there is an
open way between two sets of vertices, there is also an open path and
vice versa (see Appendix~\ref{appendix:bb} for details).
Hence, ways
can be used to determine d-separation and they make up the traversal sequence 
of the well-known Bayes-Ball algorithm~\citep{Shachter98} for testing
d-separation in linear time.
A path (or way) from $A$ to $B$ is called a \emph{back-door} (BD) path (or way) if
it starts with the edge $A
\leftarrow$. For a set of vertices $\bA$, we often speak of
\emph{proper} back-door (BD)
paths (or ways), which are such that the path $A \leftarrow \dots B$ for $A \in
\bA$ and $B \in \bB$ does not contain any other vertex in $\bA$.

Let  $G = (\bV,\bE)$ be a DAG  and let $\bI,\bR$, with  
$\bI \subseteq \bR$,  be subsets of vertices. 
Given pairwise disjoint $\bX,\bY,\bZ\subset\bV$,
set $\bZ$, with the constraint 
$\bI \subseteq \bZ \subseteq \bR$, satisfies the
\emph{front-door criterion} relative to $(\bX, \bY)$ in $G$ if~\citep{pearl1995causal}: 
\begin{description}
  \item[\textbf{FD(1).}] The set $\bZ$ intercepts all directed paths from $\bX$ to $\bY$.
  \item[\textbf{FD(2).}] There is no unblocked proper BD path from $\bX$ to $\bZ$, i.e.,
    $(\bZ \indep \bX)_{G_{\underline{\bX}}}$.
  \item[\textbf{FD(3).}] All proper BD paths from $\bZ$ to $\bY$ are blocked by $\bX$,
    i.e., $(\bZ \indep \bY \mid \bX)_{G_{\underline{\bZ}}}$.
\end{description}
The set $\bI$ consists of variables that \emph{must} be included in
the FD set, the set of variables $\bR$ consists of the ones that \emph{can} be used.
We consider graphs consisting \emph{only} of directed edges.
Bidirected edges $A \leftrightarrow B$ are frequently used to represent
confounding and can be replaced by $A \leftarrow U \rightarrow B$,
where $U$ is a new variable, which is not in $\bR$, in order to make use of the
algorithms presented here. 

The following algorithmic idea for finding a set $\bZ$ satisfying the
FD
criterion relative to $(\bX,\bY)$ (if such a set exists) was recently given by~\citet{jeong2022finding}:
\begin{enumerate}[(i)]
  \item Let $\bZ_{(\mathrm{i})} \subseteq \bR$ be the set of all variables $Z
    \in \bR$, which
  satisfy $(Z \indep \bX)_{G_{\underline{\bX}}}$.
\item Let $\bZ_{(\mathrm{ii})} \subseteq \bZ_{(\mathrm{i})}$ be the
  set of all 
$Z \in
  \bZ_{(\mathrm{i})}$, for which $\exists \bS \subseteq \bZ_{(\mathrm{i})}$ 
  s.t.~$(\{Z\} \cup \bS \indep \bY \mid
        \bX)_{G_{\underline{\{Z\} \cup \bS}}}$. 
  \item If $\bI \subseteq \bZ_{(\mathrm{ii})}$ and $\bZ_{(\mathrm{ii})}$
    intercepts all causal paths from $\bX$ to $\bY$, then output
    $\bZ_{(\mathrm{ii})}$, else output $\bot$.
\end{enumerate}

This algorithm is correct because all vertices \emph{not} in $\bZ_{(\mathrm{ii})}$
cannot be in any set satisfying the FD criterion, as they are not in
$\bR$ or would violate FD(2) and/or FD(3).
It follows from this maximality of $\bZ_{(\mathrm{ii})}$ that if $\bI$ is not a
subset of $\bZ_{(\mathrm{ii})}$ or if $\bZ_{(\mathrm{ii})}$ does not
satisfy FD(1), then no set does
(if some set $\bZ$ satisfies FD(1), then any superset does as well).
\citet{jeong2022finding} showed that step (i) can be performed in time $O(n
(n+m))$, step (ii) in time $O(n^3(n+m))$ and step (iii) in time $O(n+m)$. 

\section{A Linear-Time Algorithm for Finding Front-Door Adjustment Sets}
\label{section:lintimefinding}
In this section, we show how step (i) and (ii) can be performed in time $O(n+m)$, which leads
to the first linear-time algorithm for finding front-door adjustment sets.
For (i), the Bayes-Ball algorithm~\citep{Shachter98} can be used. It
computes all variables d-connected to a given set of variables in time
$O(n+m)$. As the Bayes-Ball algorithm forms the basis for many parts
of this work, we provide a brief introduction in
Appendix~\ref{appendix:bb}.
\begin{lemma}\label{lem:find:z1}
  It is possible to find $\bZ_{(\mathrm{i})} \subseteq \bR$, i.e., all vertices $Z$ 
  in $\bR$ with $(Z \indep
  \bX)_{G_{\underline{\bX}}}$, in time $O(n+m)$.
\end{lemma}


\begin{proof}
  Start Bayes-Ball~\citep{Shachter98} (Algorithm~\ref{alg:bb} in
  Appendix~\ref{appendix:bb}) at $\bX$ in the DAG $G_{\underline{\bX}}$. Precisely the
  vertices $\bN$ not
  reached by the algorithm satisfy $(Z \indep \bX)_{G_{\underline{\bX}}}$. Hence,
  $\bZ_{(\mathrm{i})}  = \bN \cap \bR$.
\end{proof}


\tikzset{
    old inner xsep/.estore in=\oldinnerxsep,
    old inner ysep/.estore in=\oldinnerysep,
    double circle/.style 2 args={
        circle,
        old inner xsep=\pgfkeysvalueof{/pgf/inner xsep},
        old inner ysep=\pgfkeysvalueof{/pgf/inner ysep},
        /pgf/inner xsep=\oldinnerxsep+#1,
        /pgf/inner ysep=\oldinnerysep+#1,
        alias=sourcenode,
        append after command={
        let     \p1 = (sourcenode.center),
                \p2 = (sourcenode.east),
                \n1 = {\x2-\x1-#1-0.5*\pgflinewidth}
        in
            node [inner sep=0pt, draw, circle, line width = 0.8pt, minimum width=2*\n1,at=(\p1),#2] {}
        }
    },
    double circle/.default={0.5pt}{blue!80!white},
    allowed node/.style={circle, line width = 0.8pt, draw=green!80!black},
    FD node/.style={double circle}
}

\begin{figure}
  \centering
  \begin{tikzpicture}[xscale=1.2]
      \node (x) at (0,0) {$X$};
      \node[allowed node, FD node] (a) at (1,0) {$A$};
      \node[allowed node] (b) at (2,0) {$B$};
      \node (c) at (3,0) {$C$};
      \node (y) at (4,0) {$Y$};
      \node[allowed node, FD node] (d) at (2.5,-1) {$D$};
      \node (u1) at (1.5,1) {$U_1$};
      \node (u2) at (1.5,-1) {$U_2$};
      \node (u3) at (3.5,1) {$U_3$};
      \node (u4) at (3.5,-1) {$U_4$};
      \graph[use existing nodes, edges = {arc}] {
        x -- a;
        a -- b;
        b -- c;
        c -- y;
        d -- a;
        d -- u4;
        u1 -- x;
        u1 -- y;
        u2 -- x;
        u2 -- c;
        u3 -- b;
        u3 -- y;
        u4 -- y;
      };
  \end{tikzpicture}
  \caption{Running example for the algorithms for finding FD sets in
    $O(n+m)$ given in this section. Nodes in $\bZ_{(\mathrm{i})}$ are
  marked green and nodes in $\bZ_{(\mathrm{ii})}$ are marked blue.}
  \label{fig:examples}
\end{figure}
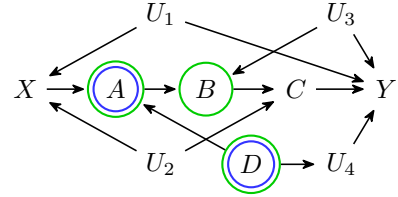
We exemplify the algorithms in this section with the running example in
Fig.~\ref{fig:examples}. Here, the goal is to find an FD set relative to $X$ and $Y$. The variables $U_1$ to $U_4$ are unobserved,
hence $\bR = \{A, B, C, D\}$. Moreover, we have $\bI = \emptyset$. 
In the graph, vertex $C$ is reachable via the
BD path $X \leftarrow U_2 \rightarrow C$ from $X$, whereas the
remaining vertices in $\bR$, $A$, $B$ and $D$, are not reachable by such a path. Hence, $\bZ_{(\mathrm{i})} = \{A,
B, D\}$. 

It remains to show how to execute step (ii), i.e., to compute 
$\bZ_{(\mathrm{ii})}$, in time $O(n+m)$.
Thus, it is our task, given a set $\bZ_{(\mathrm{i})} \subseteq \bR$ disjoint
with $\bX$ and $\bY$, which
contains all vertices satisfying (i), to decide for every $Z \in
\bZ_{(\mathrm{i})}$ whether there
exists a set $\bS \subseteq \bZ_{(\mathrm{i})}$ with $(\{Z\} \cup \bS \indep \bY \mid
\bX)_{G_{\underline{\{Z\}
\cup \bS}}}$. 

First, we define the notion of a \emph{forbidden} vertex $v$:
\begin{definition}\label{def:forbidden}
  A vertex $V$ is \emph{forbidden} if it is not in $\bZ_{(\mathrm{ii})}$. Hence, by definition
  this is the case if
  (a) $V \not\in
  \bZ_{(\mathrm{i})}$ or (b) there exists no $\bS \subseteq \bZ_{(\mathrm{i})}$
  for $V$ such that $(\{V\} \cup \bS \indep \bY
  \mid \bX)_{G_{\underline{\{V\} \cup \bS}}}$.
\end{definition}

Our goal will be to find all \emph{forbidden} vertices. The remaining vertices  then
make up the sought after set $\bZ_{(\mathrm{ii})}$. We utilize the following
key lemma. 


\begin{lemma}\label{lemma:forbidden}
  Let $G$ be a DAG and $\bX$, $\bY$ disjoint sets of vertices. Vertex $V$ is \emph{forbidden} if, and only if, 
  \begin{description}
  	\item[(A)] $V \not\in  \bZ_{(\mathrm{i})}$,  
	\item[(B)] $V\gets \bY$, or
	\item[(C)] 
  there exists an open BD way $\pi$ (consisting of at least three
  variables) from $V$ to $\bY$ given $\bX$,
  and all its nonterminal vertices are forbidden.
  \end{description}
\end{lemma}

\begin{proof}
  We show two directions. First, if $V \in \bZ_{(\mathrm{i})}$, there is no
  edge $V \leftarrow \bY$ and there exists no open
  BD way with only forbidden vertices from $V$ to $\bY$ given $\bX$, then
  there exists a set $\bS \subseteq
  \bZ_{(\mathrm{i})}$, for which
  $(\{V\} \cup \bS \indep \bY \mid \bX)_{G_{\underline{\{V\} \cup \bS}}}$ holds. It can be
  constructed by choosing, for every open way from $V$ to $\bY$, a non-forbidden vertex $W$ and
  its set $\bS_W$ (which fulfills $(\{W\} \cup \bS_W \indep \bY \mid
  \bX)_{G_{\underline{\{W\}
  \cup \bS_W}}}$) and taking the union $\bigcup_W (\{W\} \cup \bS_W) = \bS$. 
  In particular, by taking $W$ into $\bS$, we close the open way it is on as
  $W$ is a non-collider (it is not in $\bX$ as it is non-forbidden by definition) and hence the way is
  cut, due to  the removal of outgoing edges from $W$. By adding all vertices
  in $\bS_W$, it holds that $W$
  and the vertices in $\bS_W$ have no open BD way to $\bY$ given $\bX$. For
  this, note 
  that if $(\{W\} \cup \bS_W
  \indep \bY \mid \bX)_{G_{\underline{\{W\} \cup \bS_W}}}$ is true, then we also have for
  every set $\bS' \supseteq \{W\} \cup \bS_W$ that $(\{W\} \cup \bS_W \indep \bY \mid
  \bX)_{G_{\underline{\bS'}}}$. Hence, taking the union $\bigcup_{W} (\{W\} \cup
  \bS_W)$ will not open any previously closed ways. 

  Second, if $V$ is not in $\bZ_{(\mathrm{i})}$, then $V$ is forbidden by
  Definition~\ref{def:forbidden} (part \emph{(a)}) and if
  there exists an open BD way $\pi$ with only forbidden
  vertices, a set $\bS$
  satisfying \emph{(b)} can never be found (the way could only be closed by adding one
  of its vertices to $\bS$, but all of them are forbidden).
\end{proof}

\IncMargin{.5em}
\begin{algorithm}[!htp] 
  \DontPrintSemicolon
  \SetKwInOut{Input}{input}\SetKwInOut{Output}{output}
  \Indmm
  \Input{DAG $G = (\bV,\bE)$ and sets $\bX$, $\bY$, $\bZ_{(\mathrm{i})}
  \subset \bV$.}
  \Output{Set $\bZ_{(\mathrm{ii})}$.}
  \Indpp
  \SetKwFunction{FVisit}{visit\!}
  \SetKwProg{Fn}{function}{}{end}

  Initialize $\mathrm{visited}[V,\mathrm{inc}]$,
  $\mathrm{visited}[V,\mathrm{out}]$ and $\mathrm{continuelater}[V]$
  with \texttt{false} for all $V \in \bV$. \;
  $\mathrm{forbidden}[V] := \texttt{true}$ if $V \not\in
  \bZ_{(\mathrm{i})}$ else \texttt{false}. \; \label{line:setforbidden}

  \Fn{\FVisit{$G,V,\mathrm{edgetype}$}}{
    $\mathrm{visited}[V, \mathrm{edgetype}] := \texttt{true}$ \;
    $\mathrm{forbidden}[V] := \texttt{true}$ \;
    
    \If{$V \not\in \bX$\label{line:childif}}{
      \ForEach{$W \in \Ch(V)$}{
        \lIf{$\textbf{\emph{not }} \mathrm{visited}[W, \mathrm{inc}]$\label{line:visitchild}}{
          \FVisit{$G,W,\mathrm{inc}$}
        }
      }
    \If{$
      \mathrm{edgetype} = \mathrm{out} 
    $\label{line:parentif}}{
    \ForEach{$W \in \Pa(V)$}{
    
      \uIf{\textbf{\emph{not}} $\mathrm{visited}[W,\mathrm{out}]$\label{line:parentforbiddenif}}{
      	\uIf{$\mathrm{forbidden}[W]$}{\label{line:parentforbiddenif}
        \FVisit{$G,W,\mathrm{out}$} \;
      }
       \Else{$\mathrm{continuelater}[W] := \texttt{true}$ \label{line:parentnotforbiddenif}}
       }
%
    }
  }}
  \lIf{$\mathrm{continuelater}[V]$ \textbf{\emph{and not}}
  $\mathrm{visited}[V,\mathrm{out}]$\label{line:continue}}{
    \FVisit{$G,V,\mathrm{out}$}
  }
  }
  \ForEach{$Y \in \bY$}{
    \lIf{$\textbf{\emph{not }}\mathrm{visited}[Y, \mathrm{out}]$}{
      \FVisit{$G,Y,\mathrm{out}$}
    }
  }

  \Return $\bV \setminus \{V \mid \mathrm{forbidden}[V] =
  \texttt{true}\}$\;
  \caption{Finding the set $\bZ_{(\mathrm{ii})}$ in time $O(n+m)$. 
  }
  \label{alg:findingzab}
\end{algorithm}
\DecMargin{.5em}
In Fig.~\ref{fig:examples}, $X$, $Y$, all $U_i$ and $C$
are forbidden as they are not in $\bZ_{(\mathrm{i})}$ (condition (A) of
Lemma~\ref{lemma:forbidden}). Moreover, vertex $B$ is forbidden, as
there is a BD way from it via forbidden vertex $U_3$ to $Y$ (condition
(C) of Lemma~\ref{lemma:forbidden}). In
contrast, $A$ and $D$ are not forbidden and it
follows that $\bZ_{(\mathrm{ii})} = \{A, D\}$.

Lemma~\ref{lemma:forbidden} suggests an algorithmic approach for finding all forbidden vertices in
linear time (a formal description is given in
Algorithm~\ref{alg:findingzab}). First, we mark all vertices $V \not\in
\bZ_{(\mathrm{i})}$ as forbidden (line~\ref{line:setforbidden}). We then start a
graph search, similar to Bayes-Ball, at $\bY$ visiting only forbidden
vertices. Each vertex is visited at most twice, once through an incoming
and once through an outgoing edge. For this, when handling a vertex $V$, we iterate over those of its neighbors,
which could extend the path over forbidden vertices with which $V$ was
reached. As in Bayes-Ball,
this depends on the direction of the edges (collider/non-collider) and
membership in $\bX$. If a neighbor $W \not\in \bX$ is a child of $V$, we visit it
and it is consequently marked forbidden (line~\ref{line:visitchild}),
as we either have (B) that $W \leftarrow \bY$ (if $V \in \bY$) or (C) an open \emph{BD
way} from $W$ over $V$ to $\bY$ via forbidden vertices. If $W \not\in \bX$ is a parent of
$V$ and already marked 
forbidden (line~\ref{line:parentforbiddenif}),
we visit it (if it has not already been visited through an outgoing edge 
$W\rightarrow$). If it is not marked forbidden
(line~\ref{line:parentnotforbiddenif}), we record in
$\mathrm{continuelater}[W]$ that it is
connected to $\bY$ via a \emph{non-BD way} over forbidden vertices. However,
we do not visit it directly, as our search only visits forbidden
vertices. If it later becomes forbidden, we
continue the search at that point (line~\ref{line:continue}).
Note that we do not need to consider the case $W \in \bX$, as is shown
in the proof \footnote{We defer some proofs to the appendix.
Theorem~\ref{thm:findingzab} is proved in Appendix~\ref{appendix:missingproofssec3}. } 
of Theorem~\ref{thm:findingzab}.


\begin{theorem}\label{thm:findingzab}
  Given a DAG and sets $\bX, \bY, \bZ_{\mathrm{(i)}}$, Algorithm~\ref{alg:findingzab} computes the set $\bZ_{(\mathrm{ii})}$ in time $O(n+m)$.
\end{theorem}

For the graph in Fig.~\ref{fig:examples}, first the vertices not in
$\bZ_{(\mathrm{i})}$, i.e., $X$, $Y$, $U_1$ to $U_4$ and $C$, are
marked forbidden. The graph search starts at $Y$ and visits $U_3$ and
$U_4$ as they are parents of $Y$ \emph{and} already marked forbidden.
Vertex $B$ is visited as child of $U_3$ and thus marked forbidden as
well. In contrast $D$, as parent of $U_4$ is not visited and hence not marked
forbidden (it is merely marked in $\mathrm{continuelater}$). This
corresponds to the fact that there is no set $\bS$ for which $(\{B\}
\cup \bS \indep \{Y\} \mid \{X\})_{G_{\underline{\{B\}\cup \bS}}}$
holds, whereas $(\{D\} \indep \{Y\} \mid
\{X\})_{G_{\underline{\{D\}}}}$ is true (i.e., $\bS = \emptyset$
contradicts condition \emph{(b)} in Definition~\ref{def:forbidden}). The
search then halts and $A$ and $D$ remain as non-forbidden vertices,
which hence make up $\bZ_{(\mathrm{ii})}$.

As the remaining step of testing whether $\bZ_{(\mathrm{ii})}$ fulfills condition
FD(1) can be performed in linear-time as
well~\citep{jeong2022finding}, overall linear-time follows. The full
algorithm for obtaining an FD set $\bZ$ is given in Algorithm~\ref{alg:finding}.
\IncMargin{.5em}
\begin{algorithm}
  \DontPrintSemicolon
  \SetKwInOut{Input}{input}\SetKwInOut{Output}{output}
  \Indmm
  \Input{A DAG $G = (\bV,\bE)$ and sets $\bX$, $\bY$, $\bI\subseteq\bR\subseteq\bV$.}
  \Output{Set $\bZ$ with $\bI\subseteq\bZ\subseteq\bR$ or $\bot$.}
  \Indpp
  \SetKwFunction{FVisit}{visit}
  \SetKwProg{Fn}{function}{}{end}

  Start the Bayes-Ball~\citep{Shachter98} algorithm at $\bX$ in
  $G_{\underline{\bX}}$. Let $\bN$ be the non-visited vertices.
  
  Compute $\bZ_{(\mathrm{i})}  := \bN \cap \bR$ and $\bZ_{(\mathrm{ii})}$ by Algorithm~\ref{alg:findingzab}.
  
  Start a depth-first search (DFS) at $\bX$ following only directed edges,
  which stops at vertices in $\bZ_{(\mathrm{ii})}\cup \bY$. Let $\bW$ be the set of visited vertices.

  \uIf{$\bI \subseteq \bZ_{(\mathrm{ii})}$ \textbf{\emph{and}} $\bW \cap \bY = \emptyset$}{
    \Return $\bZ_{(\mathrm{ii})}$
  }
  \Else{\Return{$\bot$}}
  
  \caption{Finding an FD set $\bZ$ relative to $(\bX,\bY)$ in time $O(n+m)$.}
  \label{alg:finding}
\end{algorithm}
\DecMargin{.5em}

\begin{theorem}\label{thm:findcorrect}
Given a DAG and sets $\bX$ and $\bY$,
  Algorithm~\ref{alg:finding} finds an FD set $\bZ$ relative to $(\bX,\bY)$ with $\bI\subseteq\bZ\subseteq\bR$, or decides that such a set does not exist, in time $O(n+m)$.
\end{theorem}
\begin{proof}
According to Lemma~\ref{lem:find:z1} and Theorem~\ref{thm:findingzab},
the algorithm correctly computes $\bZ_{(\mathrm{i})}$ and
$\bZ_{(\mathrm{ii})}$. The algorithm 
  verifies the conditions that $\bI \subseteq \bZ_{(\mathrm{ii})}$
  and, by following directed edges, that $\bZ_{(\mathrm{ii})}$
    intercepts all causal paths from $\bX$ to $\bY$. As
    \citet{jeong2022finding} have shown, this makes
    $\bZ_{(\mathrm{ii})}$ an FD set.
\end{proof}

Revisiting Fig.~\ref{fig:examples}, the set $\bZ_{(\mathrm{ii})} = \{A,
D\}$ indeed blocks all causal paths between $X$ and $Y$ and is hence a
valid FD set. This can be checked by starting a
DFS at $X$, which follows the edge $X \rightarrow A$ and
does not continue from there as $A \in \bZ_{(\mathrm{ii})}$. As $X$
has no other children, the search terminates.

It follows from Theorem~\ref{thm:findcorrect}, that given a DAG and sets
$\bX$, $\bY$, and $\bZ$, one can verify in linear-time whether the set
$\bZ $ is an
FD set $\bZ$ relative to $(\bX,\bY)$ by setting
$\bI=\bR=\bZ$ and calling Algorithm~\ref{alg:finding} as it finds
$\bZ$  if and only if $\bZ$ is an FD set.

Moreover, with standard techniques developed by~\citet{van2014constructing} and
applied by~\citet{jeong2022finding} to the FD criterion, it is
possible to enumerate all FD sets with delay $O(n  \cdot
\mathrm{find}(n,m))$, where $\mathrm{find}(n,m)$ is the time it takes
to find an FD set in a graph with $n$ vertices and $m$ edges. Hence,
with the linear-time Algorithm~\ref{alg:finding} for finding FD sets presented in this work, an $O(n (n + m))$ delay enumeration algorithm follows directly. This improves the previous best run time of $O(n^4 (n+m))$ by~\citet{jeong2022finding} by a factor of $n^3$.

\begin{corollary}
  There exists an algorithm for enumerating all front-door adjustment sets in a
  DAG $G = (\bV, \bE)$ with delay $O(n (n+m))$.
\end{corollary}

A linear-time delay enumeration algorithm appears to be out-of-reach because even for the
simpler tasks of enumerating d-separators and back-door adjustment sets, the
best known delay is again $O(n (n+m))$~\citep{van2014constructing}.

\newcolumntype{L}{>{$}l<{$}} 

\def\trueInTable{\text{always}}
\def\falseInTable{\text{---}}
\def\tcontinue{\textrm{continue}}
\def\tyield{\textrm{yield}}

\def\bZXY{\bZ_{\text{XY}}}
\def\bZXYP{\bZ^{\ref{alg:finding-minimal}}_{\text{XY}}}
\def\bZXYT{\bZ^{\ref{alg:finding:minimal:tabled}}_{\text{XY}}}
\def\bZZY{\bZ_{\text{ZY}}}
\def\bZZYP{\bZ^{\ref{alg:finding-minimal}}_{\text{ZY}}}
\def\bZZYT{\bZ^{\ref{alg:finding:minimal:tabled}}_{\text{ZY}}}

\section{Finding Minimal Front-Door\\ Adjustment Sets}
\label{section:minimal}
The method in the previous section guarantees us to find an FD set (if
it exists) in linear-time. The set it will return, however, is
the maximum size FD set, as the least amount of variables are pruned
from $\bR$ in order to satisfy condition FD(2) and FD(3). 
From a practical point-of-view using this
set for FD adjustment appears artificial and impedes the
evaluation of the FD formula.
In this section, we discuss the problem of finding FD
adjustment sets of small size. More precisely we aim to find \emph{minimal} FD
sets, that is, sets for which no proper  subset satisfies the FD
criterion. 
{We remark that minimal FD sets are not necessarily
  of minimum size. E.g., in graph (ii) in Fig.~\ref{fig:FD:examples}
  the FD set $\{B,C\}$ is minimal, but $\{A\}$ is also an FD set and
has smaller sizer.  It is an open problem to efficiently compute
minimum size FD sets.} 

The obvious algorithm for finding a minimal FD set is to find a
non-minimal set and remove vertices one-by-one until no more vertices can be
removed. This trivial approach has been successfully used to find
minimal back-door adjustment sets in polynomial time
\citep{van2019separators}, but it is not applicable to FD sets.
Figure~\ref{fig:FD:examples} (iii) shows a DAG with a front-door
adjustment set $\bZ=\{A, B, C\}$.  The BD path $B \gets D \to C
\to Y$ does not violate condition FD(3), since it is not proper (in
other words, $(\{A, B, C\} \indep \{Y\} \mid \{X\})_{G_{\underline{\{A,
B, C\}}}}$ holds as the path disconnects without the outgoing edges
from $C$, which are ignored in FD(3) when $C \in \bZ$). 
Removing $C$ from the FD set turns it into a
unblocked proper path, hence $\{A,B\}$ is not an FD set. Neither
is $\{A,C\}$ nor $\{B,C\}$. Since no single vertex can be removed from $\bZ$, one
might believe $\bZ$ to be minimal. However, the only minimal
FD set is $\{A\}$.

While it would be possible to use a modified version of this
strategy, by iteratively removing a variable $W$ from the non-minimal set $\bZ$ if there
exists an FD set with $\bR = \bZ \setminus \{W\}$, a statement which could be
checked for each variable using Algorithm~\ref{alg:finding}, this would yield a time
complexity of $O(n(n+m))$. In this section, we present a linear-time
$O(n+m)$ algorithm, which moreover reveals structural insights of
(minimal) FD sets. 
We begin with a formal definition:
\begin{definition}
An FD set $\bZ$ relative to $(\bX,\bY)$ is
$\bI$\emph{-inclusion  minimal}, if and only if, no proper subset $\bZ'\subset\bZ$ with $\bI\subseteq\bZ'$ is 
an FD set relative to $(\bX,\bY)$.
If $\bI = \emptyset$, we call the set \emph{(inclusion)
minimal}.\footnote{In the majority of cases, we only speak of
\emph{minimal FD sets} and omit the word inclusion for $\bI =
\emptyset$.}
\end{definition}

The following lemma allows us to characterize minimal FD sets:
\begin{lemma}\label{lem:characterization:minimal}
Let $G$ be a DAG and  $\bX$, $\bY$ be disjoint sets of vertices.
An FD set $\bZ$ relative to $(\bX,\bY)$ is $\bI$-inclusion  minimal  if it can be written as $\bZ = \bI \cup \bZXY \cup \bZZY$ such that
\begin{enumerate}
\item For each $Z\in\bZXY$, there exists a directed proper path from $\bX$ to $Z$ to $\bY$ containing no vertex of $\bZ\setminus \{Z\}$, and
\item for each $Z\in\bZZY$, there exists a proper BD way from
  $\bI\cup\bZXY$ to $\bY$ containing an edge $Z \to$  (i.e., an
  edge facing in the direction of $\bY$) but no edge $Z' \to$ for $Z' \in \bZ\setminus \{Z\}$ which contains no vertex of $\bX$ and each collider is opened by $\bZ$. 
\end{enumerate}
\end{lemma}

\begin{proof}
Suppose there exists a smaller FD set $\bZ' \subset \bZ$ with $\bI\subseteq \bZ'$.  
Let $Z$ be a vertex of $\bZ \setminus \bZ'$. 
If $Z$ is in $\bZXY$, $\bZ'$ does not block all directed paths from $\bX$ to $\bY$ and is no FD set.

So $\bZXY \subseteq \bZ'$ and $Z$ is in $\bZZY$. 
Let $\pi$  be the BD way from $\bI \cup \bZXY$ to $\bY$
containing the edge $Z \to$ from point~2. 
Let $Z_l$ be the last vertex of $\bZ'$ on $\pi$. It exists since $\pi$ starts in $\bI\cup\bZXY\subseteq\bZ'$. It occurs with an incoming edge $Z_l\gets$ since $Z\to$ is the only outgoing edge $\bZ\to$ on $\pi$. 
So the sub-way 
of $\pi$ from $Z_l$ to $\bY$ exists in $G_{\underline{\bZ'}}$, is a BD way, and is not blocked by $\bX$.
Hence, $\bZ'$ is no FD set.
\end{proof}

The intuition behind Lemma~\ref{lem:characterization:minimal} is to
include only vertices that are necessary for identifying the effect of
$\bX$ on $\bY$. The set $\bZXY$ is needed to block causal paths
between $\bX$ and $\bY$, while the set $\bZZY$ is needed to
disconnect BD paths from $\bZXY$ to $\bY$.
The conditions are chosen, such that each path is blocked by exactly
one vertex, so no vertex from $\bZXY \cup \bZZY$ can be
removed from the FD set without opening a path.
Algorithm~\ref{alg:finding-minimal} shows how sets $\bZXY$
and $\bZZY$ can be computed, by first calling
Algorithm~\ref{alg:finding} and constructing $\bZ_{\text{An}}$ (defined in the
algorithm).

\IncMargin{.5em}
\begin{algorithm}[btp]
  \DontPrintSemicolon
  \SetKwInOut{Input}{input}\SetKwInOut{Output}{output}
  \Indmm
  \Input{A DAG $G = (\bV,\bE)$ and sets $\bX$, $\bY$, $\bI\subseteq\bR\subseteq\bV$.}
  \Output{Minimal FD set $\bZ_{\mathrm{min}}$ with $\bI\subseteq \bZ_{\mathrm{min}} \subseteq\bR$ or $\bot$ if no FD set exists.}
  \Indpp
  \SetKwFunction{FVisit}{visit}
  \SetKwProg{Fn}{function}{}{end}

 \BlankLine 

Compute the set $\bZ_{(\mathrm{ii})}$ with Algorithm~\ref{alg:finding}; If  it outputs $\bot$, then return $\bot$ and stop.

Let $\bZ_{\text{An}} \subseteq \bZii \cap \An(\bY) $ be a maximal set such that each $V \in \bZ_{\text{An}}$ is a parent of $\bY$ or there exists a directed path from $V$ to $\bY$ 
containing no vertex of $\bX\cup (\bZii\setminus \{V\})$.

Let $\bZXY \subseteq \bZ_{\text{An}} \cap \De(\bX)$ be a maximal set such that, for each $V \in \bZXY$, there exists a directed proper path from $\bX$ to $V$ 
containing no vertex of $\bI\cup(\bZ_{\text{An}}\setminus \{V\})$. 

Let $\bZZY \subseteq \bZ_{\text{An}}$ be a maximal set such that, for
each $V \in \bZZY$, there exists a proper BD way from $\bI\cup\bZXY$ to $V$ containing no edge $Z_{\text{An}}\to$ with $Z_{\text{An}}\in\bI\cup\bZ_{\text{An}}$ and no vertex of $\bX$, and all colliders are in $\bI\cup\bZ_{\text{An}}$.  

\Return $\bZ_{\mathrm{min}} := \bI\cup \bZXY \cup \bZZY$

  \caption{Finding an \emph{$\bI$-inclusion minimal} front-door
  adjustment set $\bZ$ 
  in linear time.}
  \label{alg:finding-minimal}
\end{algorithm}
\DecMargin{.5em}

\begin{theorem}\label{thm:finding:minimal}
Given a DAG $G$ and sets $\bX$ and $\bY$,
Algorithm~\ref{alg:finding-minimal}
finds an  $\bI$-inclusion minimal FD set $\bZ$ relative to $(\bX,\bY)$, with $\bI\subseteq\bZ\subseteq\bR$, or decides that such a set does not exist, in linear time.

\end{theorem}

\begin{figure}
  \centering
  \begin{tikzpicture}[xscale=1.25]
      \node (x) at (5,0) {$X$};
      \node (a) at (6,0) {$A$};
      \node (b) at (7,0) {$B$};
      \node (c) at (8,0) {$C$};
      \node (y) at (9,0) {$Y$};
      \node (d) at (7,-1) {$D$};
      \node (e) at (8.25,-1) {$E$};
      \node (u1) at (7,1) {$U_1$};
      \node (u2) at (6,-1) {$U_2$};
      \graph[use existing nodes, edges = {arc}] {
        x -- a;
        a -- b;
        b -- c;
        c -- y;
        d -- y;
        e -- d;
        e -- y;
        u2 -- a;
        u2 -- d;
        u1 -- x;
      };
      \draw[arc] (b) to[bend left] (y);
      \draw[arc] (u1) to[bend left=23] (y);
  \end{tikzpicture}
  \caption{Example graph for finding a minimal FD set.}
  \label{fig:example:min}
\end{figure}
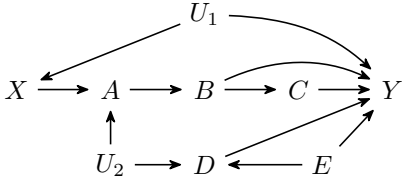
These concepts are exemplified in Fig.~\ref{fig:example:min},
where the goal is to find a minimal FD set with respect to $X$ and~$Y$
with $\bR = \{A, B, C, D, E\}$ and $\bI = \emptyset$.
The FD set
$\bZ_{(\mathrm{ii})} = \{A, B, C, D, E\}$ is returned by
Algorithm~\ref{alg:finding}. Vertices $B$, $C$, $D$, and $E$ are in
$\bZ_{\text{An}}$ as they are parents of $\bY$ and in $\bZ_{(\mathrm{ii})}$. Based on
this set, we first have $\bZXY = \{B\}$ as every path
from $X$ to $C$ goes through $B$ and there is no path to $D$ and $E$.
Moreover, $\bZZY = \{D, E\}$ as both are needed to block
BD paths from $B$ to $Y$. Notably, $D$ alone does not suffice as there would be
the BD path $D \leftarrow E \rightarrow Y$ violating FD(3), i.e., the
statement $(\bZ \indep \bY \mid \bX)_{G_{\underline{\bZ}}}$. By
including $E$, this condition is satisfied because outgoing edges from
variables in $\bZ$ are removed in $G_{\underline{\bZ}}$.
Hence $\bZ_{\mathrm{min}} = \bZXY \cup \bZZY = \{B,
D, E\}$ is a minimal FD set.

From this theorem, we can conclude that the conditions of Lemma~\ref{lem:characterization:minimal} are  "if and only if" conditions:

\begin{corollary}\label{lem:characterization:minimal:reverse}
An FD set $\bZ$ relative to $(\bX, \bY)$ is
$\bI$-inclusion minimal if \emph{and only if} it can be written as 
$\bZ = \bI \cup \bZXY \cup \bZZY$ such that
$\bZXY$ and $\bZZY$ satisfy conditions 1. and 2. given in
Lemma~\ref{lem:characterization:minimal}.
\end{corollary}

\begin{proof}
The "if"-direction is Lemma~\ref{lem:characterization:minimal}. The reverse follows from the algorithm  of Theorem~\ref{thm:finding:minimal}.
For a given set $\bZ$, apply the algorithm 
with $\bR=\bZ$. The algorithm returns a minimal set $\bZ'\subseteq\bZ$
that satisfies the conditions of
Lemma~\ref{lem:characterization:minimal}. If $\bZ$ is minimal,
$\bZ=\bZ'$ and $\bZ$ satisfies the conditions, too.
\end{proof}

\begin{corollary}
An $\bI$-inclusion  minimal FD set $\bZ$ relative to $(\bX,\bY)$ is a subset of $\bI\cup\An(\bY)$.
\end{corollary}

As we show empirically in Appendix~\ref{appendix:further:experiments}, the minimal FD sets returned
by Algorithm~\ref{alg:finding-minimal} are often much smaller compared
to the
maximal FD set. However, we reemphasize that they are not
guaranteed to have \emph{minimum} size, i.e., there might exist an
FD set of smaller cardinality, and that it is open whether
it is possible to
compute such sets efficiently.

\section{Experiments}
We have shown that the run time of the algorithms proposed in this work
scales linearly in the size of the graph. In this
section, we empirically show that this translates to practical
implementations, which are able to handle hundreds of thousands of
variables. 

\begin{figure*}
  \begin{center}
    \includegraphics[width=\textwidth]{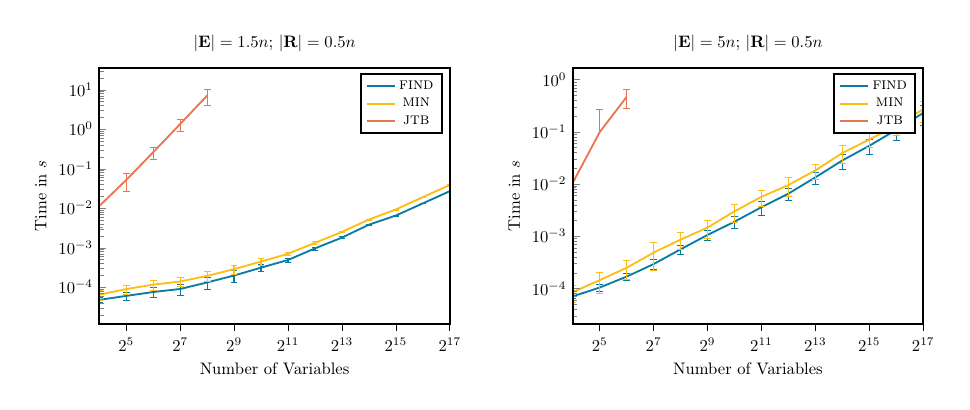}
\end{center}
  \caption{
    Log-Log plot of the average run time in seconds for
    \citep{jeong2022finding} (\textsc{jtb}),
    Algorithm~\ref{alg:finding} (\textsc{find}) and
    Algorithm~\ref{alg:finding-minimal} (\textsc{min})
    on Erd\H{o}s-R\'enyi graphs with $1.5n$ (left) and $5n$
    (right) edges, corresponding to expected vertex degree $3$ and $10$. 
    $|\bX|, |\bY|$ are random integers between $1$ and $3$;
    $|\bI| = 0$ and $|\bR| =
  0.5n$. For each choice of $n$, we
average over 50 graphs. 
}
  \label{fig:experiments}
\end{figure*}

We implement our methods in Python, Julia, and JavaScript to enable wide
practical usage in the causal inference community.\footnote{To
  facilitate adoption, we
  base on the \texttt{networkx} Python package, which is the
  graph library used by causal inference packages
  such as \texttt{DoWhy}~\citep{dowhy}. Similarly, in Julia we
  allow for easy integration into 
  \texttt{CausalInference.jl}~\citep{causalinferencejl} and in Javascript into the
  \texttt{DAGitty} environment~\citep{textor2016robust}.
} 
In the main paper, we show the run time for the Python implementation
and compare it to the author's implementation
of the algorithm for finding FD sets given
in~\citet{jeong2022finding} (\textsc{jtb}). In
Appendix~\ref{appendix:further:experiments}, we also
evaluate our Julia and Javascript implementations and provide more detailed results.
We ran the experiments on a single core of the AMD
Ryzen Threadripper 3970X 32-core processor on a 256GB
RAM machine.\footnote{The experiments can be replicated within a few
  days on a desktop
computer without any problems.} 
Figure~\ref{fig:experiments} shows the average run time of the
algorithms in seconds. Per instance, each algorithm was given a time
limit of 30 seconds (we only report results for parameter choices for
which each instance was solved within the allocated time).

The results confirm our theoretical findings as \textsc{find}
(Algorithm~\ref{alg:finding}) outperforms \textsc{jtb} by a factor of
more than $10^4$ on medium-sized instances (the gap widens with the number of
variables). For the sparser graphs (expected degree $3$), \textsc{jtb}
does not terminate in under 30 seconds in case of more than
256 vertices; for denser graphs (expected degree $10$) the performance
degrades faster and only the instances up to $64$ vertices are solved
within the time limit. In the Appendix~\ref{appendix:further:runtimes}, we also provide
a run-time comparison for the real-life DAGs from the
\texttt{bnlearn} repository~\cite{scutaribnlearn}, which further validate our findings, showing that
\textsc{find} yields a sizeable improvement even for  smaller graphs.  

These gains also translate to \textsc{min} which has merely a slightly higher
run time compared to \textsc{find}, while giving significantly smaller FD
sets, as we report in more detail in
Appendix~\ref{appendix:further:experiments}. In particular, the FD
sets returned by \textsc{find} and \textsc{jtb} grow linearly in the
number of vertices, even when \textsc{min} computes FD sets of size
zero or one.

\section{Conclusions}
We have developed efficient algorithms for solving several tasks 
related to estimation of causal effects via FD adjustment, 
including finding a minimal FD set in asymptotically optimal run
time. 
Our work shows that, from the algorithmic 
perspective, these tasks are not harder than for the celebrated back-door (BD)
adjustment. We also offer implementations of our algorithms which 
significantly outperform previous methods allowing practical usage in 
empirical studies even for very large
instances.\footnote{The implementations and code to reproduce the
  experiments are available at \url{https://github.com/mwien/frontdoor-adjustment}.} 


An important open problem is whether \emph{minimum size} FD sets can be
computed efficiently. For back-door adjustment, this is the
case as there exists an $O(n^3)$ algorithm \citep{van2019separators},
while it is not clear if a polynomial-time algorithm exists in the FD
case as well.
Moreover, throughout this work, we assume that the causal DAG is known and correctly specified 
by the researcher. It would 
be interesting to combine our approach with causal discovery methods
by developing algorithms for causal models representing a class
of DAGs,  e.g., a Markov equivalence class, instead of a single one. 
In such settings, the results could provide a means to construct more ``robust'' 
front-door adjustment sets if there are several options for DAGs which can 
be learned from data.

Another issue is that FD adjustment is still not as thoroughly understood as
covariate adjustment, for which a \emph{complete} graphical characterization
exists~\citep{ShpitserVR2010}, such that a set $\bZ$ can be used to
compute the causal effect of $\bX$ on $\bY$ using the covariate
adjustment formula if, \emph{and only if}, it satisfies the adjustment
criterion. For the FD criterion, the analogous statement does not hold and we
consider it an important direction for future work to extend 
it and our algorithms in this way. 
Finally, in the classical FD criterion it has to hold that there is no
BD path from $\bX$ to $\bZ$. In practice, this can be rather
restrictive. 
Recent works have relaxed this
assumption~\citep{hunermund2019causal,fulcher2020robust} 
by instead demanding that a set $\bW$ exists, which blocks
all such BD paths (FD(3) is generalized in the same way) and this
is also an interesting avenue for future research.

\section*{Acknowledgements}
This work was supported by the Deutsche Forschungsgemeinschaft (DFG)
grant 471183316 (ZA 1244/1-1).

\bibliography{main.bib}

\clearpage

\appendix

\onecolumn

\begin{center}
  \textbf{\huge Appendix}
\end{center}
\vspace*{1cm}

\setcounter{algocf}{3}
\setcounter{lemma}{3}
\setcounter{figure}{4}

\section{Introduction to the Bayes-Ball Algorithm}
\label{appendix:bb}
In the main paper, we often utilize the fact that testing d-separation
statements and similar tasks can be performed in linear-time. E.g., in
Lemma~\ref{lem:find:z1} it is necessary to find all variables
d-separated from $X$ in the graph with outgoing edges from $X$
removed. The underlying algorithm to achieve such results is the
well-known Bayes-Ball Algorithm~\citep{Shachter98}.

As this algorithm plays a major role this work and the more involved algorithms (such
as Algorithm~\ref{alg:findingzab} and~\ref{alg:general:visit}) are
conceptually build on top of it, we give a brief summary.

The main idea is similar to a classical graph search (in this work, we
use algorithms stylistically close to depth-first-search (DFS)).
The standard DFS is started at a certain vertex (say $X$) and a
recursive visit function is called for every neighbor, which has not
been visited before.

For comparison, we briefly state DFS in
Algorithm~\ref{alg:findingdfs} for directed graphs
under the classical definition of a path (causal paths in causal
inference terminology) and for a set of vertices $\bX$.
\begin{algorithm}
  \DontPrintSemicolon
  \SetKwInOut{Input}{input}\SetKwInOut{Output}{output}
  \Input{A DAG $G = (\bV,\bE)$ and set of vertices $\bX$.}
  \Output{Set of vertices reachable from $\bX$ via directed paths in $G$.}
  \SetKwFunction{FVisit}{visit}
  \SetKwProg{Fn}{function}{}{end}

  Initialize $\mathrm{visited}[V]$ with \texttt{false} for all $V \in \bV$. \;
  \BlankLine
  \Fn{\FVisit{$G$, $V$}}{
    $\mathrm{visited}[V] := \texttt{true}$ \;
    \ForEach{$W \in \Ch(V)$}{
      \lIf{\textbf{\emph{not}} $\mathrm{visited}[W]$}{
          \FVisit{$G$, $W$}
        }
      }
  }
  \BlankLine
  \ForEach{$X \in \bX$\label{line:startingfor}}{ 
    \lIf{\textbf{\emph{not}} $\mathrm{visited}[X]$}{\FVisit{$G$, $X$}}
  }
    \Return $\{V \mid \mathrm{visited}[V] = \texttt{true}\}$\;
  \caption{Depth-First-Search started at vertex set $\bX$. The set of all vertices
  visited during the search is returned.}
  \label{alg:findingdfs}
\end{algorithm}

Note that the search is started for each vertex $X \in \bX$ separately
in the for-loop in line~\ref{line:startingfor}, unless this vertex was already visited
during an earlier search. The latter guarantees the run time of $O(n+m)$,
because \texttt{visit} is called at most once for each vertex and hence each
edge is visited at most once. 

The Bayes-Ball algorithm is a generalization of this approach to allow
for the more complicated path definition used in d-connectedness,
given a set $\bZ$. In
particular, which neighbors of $V$ have to be considered depends on
the direction of the edge with which $V$ was reached. 
If $V$ is reached through $\leftarrow V$, then both $W$ with
$\leftarrow V \leftarrow W$ and $\leftarrow V \rightarrow W$
correspond to an open path if $V \not\in \bZ$, else the path is
closed.
If $V$ is reached through $\rightarrow V$ and $V \not\in \bZ$, then only
children $W$ with $\rightarrow V \rightarrow W$ correspond to an open
path. Conversely for parents $W$ with $\rightarrow V \leftarrow W$,
the path is closed and would only be open if $\De(V) \cup \bZ \neq
\emptyset$.

To check the latter, it would, however, be inconvenient (and costly with regard to the
run time) to check all descendants of $V$ for membership in $\bZ$. One
insight of the Bayes-Ball search is that this is not necessary.
Indeed, the algorithm will only check whether $V \in \bZ$ and if it is and
$V$ is a collider the search continues.

This suffices because assume there is such a $Z$ which is in $\bZ$ and
a descendant of $V$. This means, we have $\rightarrow V \leftarrow W$
and $V \rightarrow P_1 \rightarrow P_2 \rightarrow \dots \rightarrow P_p
\rightarrow Z$. Then Bayes-Ball will find the open \emph{way} $\rightarrow
V \rightarrow P_1 \rightarrow P_2 \rightarrow \dots \rightarrow P_p
\rightarrow Z \leftarrow P_p \leftarrow \dots \leftarrow P_2
\leftarrow P_1 \leftarrow V \leftarrow W$.

A \emph{way} allows for vertices to be visited more than once
(compared to a path). It holds that two vertices are d-connected by an
open path given $\bZ$ if, and, only if there is an open way. Here, a way
is open if every non-collider is not in $\bZ$ and every collider is in
$\bZ$ (in contrast to a path where the latter holds with regards to the
descendants of a collider). 

The implementation of Bayes-Ball is given in
Algorithm~\ref{alg:bb}. It
stores for each vertex separately whether it was visited through an
incoming and whether it was visited through an outgoing edge. As the
neighbors for continuing the search are different in the two cases, both have to be considered
and a vertex may be visited twice. More precisely, the
neighbors considered when visiting $V$ depend on the membership of $V
\in \bZ$ and the direction of the edge through which $V$ was discovered
in the manner described above.

As function \texttt{visit} is called at most twice for each vertex,
each edge is considered a constant number (more precisely, three) of
times as well. Hence, the algorithm has run time $O(n+m)$.
\begin{algorithm}
  \DontPrintSemicolon
  \SetKwInOut{Input}{input}\SetKwInOut{Output}{output}
  \Input{A DAG $G = (\bV,\bE)$ and disjoint vertex sets $\bX$ and $\bZ$.}
  \Output{Set of vertices d-connected to $\bX$ in $G$ given $\bZ$.}
  \SetKwFunction{FVisit}{visit}
  \SetKwProg{Fn}{function}{}{end}

  Initialize $\mathrm{visited}[V,\mathrm{inc}]$ and
  $\mathrm{visited}[V,\mathrm{out}]$ with \texttt{false} for all $V \in \bV$. \;

  \BlankLine

  \Fn{\FVisit{$G$, $V$, $\mathrm{edgetype}$}}{
    $\mathrm{visited}[V, \mathrm{edgetype}] := \texttt{true}$ \;
    
    \If{$V \not\in \bZ$}{
      \ForEach{$W \in \Ch(V)$}{
        \lIf{\textbf{\emph{not}} $\mathrm{visited}[W, \mathrm{inc}]$}{
          \FVisit{$G$, $W$, $\mathrm{inc}$}
        }
      }
    }
    \If{$(\mathrm{edgetype} = \mathrm{inc}$ \textbf{\emph{and}}  $V \in \bZ)$
      \textbf{\emph{or}}
      $(\mathrm{edgetype} = \mathrm{out}$ \textbf{\emph{and}} $V \not\in \bZ)
    $}{
    \ForEach{$W \in \Pa(V)$}{
      \lIf{\textbf{\emph{not}} $\mathrm{visited}[W, \mathrm{out}]$}{
        \FVisit{$G$, $W$, $\mathrm{out}$}
      }
    }
  }
}

\BlankLine

  \ForEach{$X \in \bX$}{
    \lIf{\textbf{\emph{not}} $\mathrm{visited}[X, \mathrm{out}]$}{
      \FVisit{$G$, $X$, $\mathrm{out}$\label{line:startwithout}}
    }
  }

  \Return $\{V \mid \mathrm{visited}[V, \mathrm{inc}]$
    \textbf{or} $\mathrm{visited}[V,
  \mathrm{out}] = \texttt{true} \}$\;
  \caption{Bayes-Ball Search started at vertex set $\bX$. The set of
  all vertices d-connected to $\bX$ given $\bZ$ in $G$ is returned.}
  \label{alg:bb}
\end{algorithm}

Note that it suffices to start \texttt{visit} with edgetype $\mathrm{out}$
(line~\ref{line:startwithout}), as
$\bX$ and $\bZ$ are disjoint and hence all neighbors of $X$ are considered
during this call.

In both the DFS and the Bayes-Ball algorithm, we gave a formulation
which returns all visited vertices. As utilized in
Lemma~\ref{lem:find:z1} in the main paper, this enables testing
d-separation (the statement $(\bX \indep Y | \bZ)_G$ holds
if $Y$ is not visited when starting Bayes-Ball at $\bX$). 

\section{Missing Proofs of Section~\ref{section:lintimefinding}}
\label{appendix:missingproofssec3}
We now provide the missing proof of Theorem~\ref{thm:findingzab}. It
proceeds by showing that Algorithm~\ref{alg:findingzab} correctly
labels vertices as forbidden (by correctness of
Lemma~\ref{lemma:forbidden}), which yields the statement.

\begin{proof}[Proof of Theorem~\ref{thm:findingzab}]
The correctness follows from Lemma~\ref{lemma:forbidden} and the fact that a
vertex $V$ is marked forbidden, if, and only if, (A) it is not
in $\bZ_{(\mathrm{i})}$ or (B) it is a child of $\bY$ or (C) 
there is an open BD way from $V$ to $\bY$ given $\bX$ over forbidden vertices. 
For this, we first observe that the algorithm is sound, i.e, a vertex marked
forbidden is actually forbidden. Moreover, the
completeness with regard to (A) and (B) is clear, hence, it  remains to show it with regard to (C).

Consider the situation when the algorithm
has terminated and assume there are forbidden vertices, which have not been visited by the algorithm and are not
marked forbidden (let these be $\bW
= \{W_1, \dots, W_k\}$ and note that these are forbidden due to (C), else they
would have been marked). We show that there exists an open BD way $\pi$
over forbidden vertices from some $W_i$ to $\bY$ which does not contain any other
$W_j$. If this were not the case, consider the last $W_j$ on an open
BD way from $W_i$ to $\bY$ (such a way has to exist as $W_i$ is forbidden,
it is only not \emph{marked} forbidden). If
it occurs as $W_j \leftarrow$ the existence of $\pi$ would follow.
Hence, all such $W_j$ occur
as $W_j \rightarrow$. Then it would hold that $(\bW \indep \bY \mid
\bX)_{G_{\underline{\bW}}}$ and the vertices in $\bW$ would not be forbidden. Thus, there is such a $\pi$ and it starts with the edge
$W_i \leftarrow P$ with $P$ being visited by the algorithm (on $\pi$, every vertex is marked
forbidden and it is open, hence, the graph search must have reached $P$ by correctness of
Bayes-Ball). 
Furthermore, no vertex $X\in\bX$ occurs on $\pi$. It cannot occur as
non-collider since that would block~$\pi$; and it cannot occur as
collider, since then the path from $W_j$ to $\to X$ would be a BD path from $X$ to $W_j$ and $W_j$ would already be marked as forbidden.
It follows that $W_i$ as a child of $P$ is also visited and marked
forbidden, which leads to a contradiction. 

The run time bound can be derived from the fact that for each vertex we call \texttt{visit}
at most two times (for $\mathrm{edgetype} = \mathrm{inc}$ and
$\mathrm{edgetype} = \mathrm{out}$)
and as every such call has cost $O(|\Ne(v)|)$, every edge is visited only a
constant amount of times. Hence, we obtain a run time of $O(n+m)$.
\end{proof}

\section{Missing Proofs of Section~\ref{section:minimal}}
In this section, we provide the missing proofs of
Section~\ref{section:minimal}. During this, we develop an elegant and
concise framework for implementing the various graph searches
discussed in this work.

Each step of Algorithm~\ref{alg:finding-minimal} for finding minimal
FD sets is performed by basing on
the DFS approach introduced in Section~\ref{appendix:bb}. However,
the searches have different rules which vertices can be visited and which vertices are yielded to the output set. 
Rather than repeating the description of the algorithm for each step,
we generalize the DFS presented in Algorithm~\ref{alg:findingdfs} to one general  graph search, Algorithm~\ref{alg:general:visit}, that can handle all different cases. Each step of Algorithm~\ref{alg:finding-minimal} becomes then a call to Algorithm~\ref{alg:general:visit}.

In the generalization, 
we specify the rules in the form of a table rather than implicitly as code in a "\FVisit{}" function.
For each possible combination of previous and next edge, the table lists if the vertex after the next edge should be visited. The table also lists whether that vertex should be yielded to the output. 

There are four possible combinations $\to V\to W$, $\gets V\gets W$, $\gets V\to W$, and $\to V\gets W  $ of explicit edges to consider. 
The first edge $\to V$ or $\gets V$ corresponds to the parameters of the "\FVisit{}" function. $V$ is the visited vertex and the edge the $\mathrm{edgetype}$ parameter. The second edge $V\to W$ or $V\gets W$, describes whether a parent or child $W$ of $V$ will be visited next. 

Furthermore, we consider two additional combinations of the initial
edges: When starting the search at vertices $\bX$, we visit the
children and parents of $\bX$, similarly to the {\bf foreach}{\ $X \in
\bX$} loop in Algorithm~\ref{alg:findingdfs}. Since there are no
previously visited edges into $\bX$, these are denoted as two separate
combinations  $\mathrm{init}\ V \to W$ and $\mathrm{init}\ V \gets W$, where $V \in \bX$. 

For each combination, the rule table has a row. In each row, it says,
whether the vertex $W$ is visited next, and/or the vertex $W$ is yielded to the output set. Like the DFS, the general graph search visits all vertices that can be
visited, and in this sense returns a maximal set. It also runs in linear time.

\begin{algorithm}
  \caption{General  graph search in time $O(n+m)$.}
  \label{alg:general:visit}
  \DontPrintSemicolon
  \SetKwInOut{Input}{input}\SetKwInOut{Output}{output}
  \Input{A DAG $G = (\bV,\bE)$, starting vertices $\bX$, and rule
  table $\tau$ mapping
$(\mathrm{edgetype},\mathrm{vertex},\mathrm{edgetype},\mathrm{vertex})$
to $\wp\{\mathrm{continue},\mathrm{yield}\}$. }
  \Output{Yielded vertices.}
  \SetKwFunction{FVisit}{visit}
  \SetKwProg{Fn}{function}{}{end}

  \BlankLine

  Initialize $\mathrm{visited}[V,\mathrm{inc}]$,
  $\mathrm{visited}[V,\mathrm{out}]$ and $\mathrm{result}[V]$
  with \texttt{false} for all $V \in \bV$. \;  
  \BlankLine

  \Fn{\FVisit{$G$, $V$, $\mathrm{edgetype}$}}{
  $\mathrm{visited}[V, \mathrm{edgetype}] := \texttt{true}$ \;
  \ForEach{$e \in \{\mathrm{inc},\mathrm{out}\}$}{
    
    \If{$(e = \mathrm{inc})$}{$\bN := \Ch(V)$}
    \If{$(e = \mathrm{out})$}{$\bN := \Pa(V)$}

  \ForEach{$W \in \bN$}{
    \If{$\mathrm{yield} \in \tau(\mathrm{edgetype}, V, e, W)$ }{
        $\mathrm{result}[W] := \texttt{true}$ \;
      }
      \If{\textbf{\emph{not}} $\mathrm{visited}[W, e]$
      \textbf{\emph{and}} $\mathrm{continue} \in \tau(\mathrm{edgetype}, V, e, W)$ }{
       \FVisit{$G$, $W$, $e$}\; 
      }
    }
  }
 }
 \BlankLine 
  \ForEach{$V \in \bX$}{
    \FVisit{$G$, $V$, $\mathrm{init}$} \;
  }

  \Return $\{V \mid \mathrm{result}[V] = \texttt{true}\}$ \;
\end{algorithm}

\begin{algorithm}[ht]
  \DontPrintSemicolon
  \SetKwInOut{Input}{input}\SetKwInOut{Output}{output}
  \Input{A DAG $G = (\bV,\bE)$ and vertex sets $\bX$ and $\bZ$.}
  \Output{Set of vertices d-connected to $\bX$ in $G$ given $\bZ$.}

  \BlankLine
  
Return the vertices yielded by Algorithm~\ref{alg:general:visit} starting from vertices $\bX$ with rule table\vspace{1mm}
%
%
\begin{tabular}{lll} \toprule
	case     & continue to $W$ & yield $W$ \\ \midrule
  init $V \to W$   & $\trueInTable$         & $\trueInTable$    \\
  init $V \gets W$ & $\trueInTable$  & $\trueInTable$ \\ 
	$\to V\to W$      & $V \not\in \bZ$         & $V \not\in \bZ$         \\
	$\gets V\gets W$  & $V \not\in \bZ$  & $V\not\in \bZ$             \\
	$\gets V\to W$    & $V\not\in \bZ$    & $V \not\in \bZ$              \\
	$\to V\gets W$    & $V\in \bZ$         & $V  \in \bZ$              \\ \bottomrule
\end{tabular}
\BlankLine
  \caption{Bayes-Ball Search as special case of
  Algorithm~\ref{alg:general:visit}. It returns the same set as
Algorithm~\ref{alg:bb}.}
  \label{alg:bb2}
\end{algorithm}

\begin{algorithm}
  \DontPrintSemicolon
  \SetKwInOut{Input}{input}\SetKwInOut{Output}{output}
 \DontPrintSemicolon
  \SetKwInOut{Input}{input}\SetKwInOut{Output}{output}
  \Input{A DAG $G = (\bV,\bE)$ and sets $\bX$, $\bY$, $\bZ_{(\mathrm{i})}
  \subseteq \bV$.}
  \Output{Set $\bZ_{(\mathrm{ii})}$.}
  \SetKwFunction{FVisit}{visit}
  \SetKwProg{Fn}{function}{}{end}

  \BlankLine
  
  
  $\bA :=\An(\bY)$

Let $\bZ'$ be the vertices yielded by Algorithm~\ref{alg:general:visit} starting from vertices $\bY$ with rule table\vspace{1mm}
\begin{tabular}{lll} \toprule
	case     & continue to $W$ & yield $W$ \\ \midrule
  init $V \to W$   & $\trueInTable$         & $\trueInTable$    \\
  init $V \gets W$ & $W\not\in\bZ_{(\mathrm{i})}$  & \falseInTable \\ 
	$\to V\to W$      & $V \not\in\bX$        &  $V \not\in\bX$ \\
	$\gets V\gets W$  & $V \not\in \bX \land W\not\in\bZ_{(\mathrm{i})}$  & \falseInTable \\ 
	$\gets V\to W$    & $V \not\in \bX$ & $V \not\in \bX$             \\
  $\to V\gets W$    & $V\in \bA \land W\not\in \bZ_{(\mathrm{i})}$
                    & \falseInTable              \\ \bottomrule
\end{tabular}
\BlankLine
\Return{$\bZ_{(\mathrm{i})}\setminus\bZ'$}
  \caption{Using Algorithm~\ref{alg:general:visit} to find the
  forbidden vertices. It returns the same set as
Algorithm~\ref{alg:findingzab}.}
  \BlankLine
  \label{alg:findingzab:table}
\end{algorithm}

 \begin{lemma}\label{lem:table:path}
 Given a DAG $G = (\bV,\bE)$, starting vertices $\bX$, and rule table $\tau$, Algorithm~\ref{alg:general:visit} returns a maximal vertex set $\bW$ such that there exists a way
 $v_1 e_1 v_2 e_2 \ldots e_{k-1} v_k $ where $v_i\in\bV$, $e_i\in\bE$, $v_1 \in \bX$, $v_k \in \bW$, and either
 
 \begin{itemize}
   \item $k = 2$ and $\tyield\in\tau(\mathrm{init}, v_{1}, e_{1}, v_{2})$, or
   \item $\tcontinue\in\tau(\mathrm{init}, v_1, e_1, v_2)$, $\tcontinue\in\tau(e_i, v_{i+1}, e_{i+1}, v_{i+2})$ (for $i + 2 < k$),  and $\tyield\in\tau(e_{k-2}, v_{k-1}, e_{k-1}, v_{k})$,
 \end{itemize}
 
 in time $O(n+m)$.
 \end{lemma}
 
 \begin{proof}
 The algorithm performs an obvious graph search. Each vertex is only
 visited, if there exists a way from $\bX$ to the vertex on which the
 rule table $\tau$ returns $\mathrm{continue}$ for each two-vertex
 sub-way. A vertex is only returned if $\tau$ returns $\mathrm{yield}$.
 
 The run time is $O(n+m)$, because each vertex is only visited once from each edgetype.
 \end{proof}

 Algorithm~\ref{alg:bb2} illustrates how the Bayes-Ball Algorithm~\ref{alg:bb} can be given as  a special case of Algorithm~\ref{alg:general:visit}.
 
 Algorithm~\ref{alg:findingzab:table} shows how
 Algorithm~\ref{alg:findingzab} can be implemented using
 Algorithm~\ref{alg:general:visit}. In
 Proposition~\ref{prop:finding:zab:table}, we show these two
 algorithms are equivalent. Most of the old algorithm can be directly
 translated into table rules. However, the vertices marked as
 $\mathrm{continuelater}$ need to be handled differently since the
 table rules are static and do not change, so it would not be possible to
 encode that a vertex first cannot be visited from its children before
 it is in  $\mathrm{continuelater}$, but later can be visited from a
 child after it is in $\mathrm{continuelater}$. We resolve this
 problem by utilizing the set of ancestors of $\bY$ (these have to be
 computed beforehand, making this implementation slightly less
 efficient from a practical point-of-view; the asymptotic run time is
 unaltered).
 
 \begin{proposition}\label{prop:finding:zab:table}
 Algorithm~\ref{alg:findingzab} and Algorithm~\ref{alg:findingzab:table} are equivalent.
 \end{proposition}
 \begin{proof}
 The "\FVisit{}" function of Algorithm~\ref{alg:findingzab} has two
 {\bf foreach} loops. The first loop visits all descendants of each
 visited vertex. This corresponds to the three rules of cases ending with $\to W$.
 The second loop visits all ancestors until the first non-forbidden
 vertex. This corresponds to the rules  $\mathrm{init}\ V \gets W$  and $\gets V\gets W$. 
 
 The "\FVisit{}" function only moves from an incoming to an outgoing edge when the vertex has been marked as $\mathrm{continuelater}$ (line~\ref{line:continue}). 
 This corresponds to the rule $\to V \gets W$ because only 
 the ancestors of $\bY$ can be marked as $\mathrm{continuelater}$.  
  They are only marked by the second loop, which only visits ancestors. First the ancestors of $\bY$, and then the ancestors of vertices already marked as $\mathrm{continuelater}$. Inductively, it follows that only ancestors of $\bY$ are ever marked.
 In the other direction, each ancestor $W$ of $\bY$ can be treated as
 if  it was marked as $\mathrm{continuelater}$. Either it is marked
 there, or it is visited by an incoming edge from its parent, so all
 vertices on the path from $W$ to $\bY$ will be marked as forbidden.
 Thus all parents of the vertices on that  path can be visited, and the parents of $W$, too.
 \end{proof}
 
 Algorithm~\ref{alg:finding:minimal:tabled} gives the rules to
 implement the minimal FD search. 
 Proposition~\ref{prop:minimal:algo:equivalence} shows that those rules match the conditions given in Algorithm~\ref{alg:finding-minimal}. 

 \begin{algorithm*}
  \caption{Finding a minimal front-door adjustment set $\bZ$ relative to $(\bX,\bY)$ in time $O(n+m)$. The tables give for each case of possible edge types, the conditions under which the rule table should return continue or yield.}
  \label{alg:finding:minimal:tabled}
  \DontPrintSemicolon
  \SetKwInOut{Input}{input}\SetKwInOut{Output}{output}
  \Input{A DAG $G = (\bV,\bE)$ and sets $\bX$, $\bY$, $\bI\subseteq\bR\subseteq\bV$.}
  \Output{Minimal FD set $\bZ$ with $\bI\subseteq\bZ\subseteq\bR$ or $\bot$.}
  \SetKwFunction{FVisit}{visit}
  \SetKwProg{Fn}{function}{}{end}

  \BlankLine
  Compute the set $\bZ_{(\mathrm{ii})}$ with Algorithm~\ref{alg:finding}.

  \uIf{$\bZii= \bot$}{\Return $\bot$}
 

Let $\bZ_{\text{An}}$ be given by Algorithm~\ref{alg:general:visit}
from vertices $\bY$ with rule table $\qquad$
\begin{tabular}{lll} \toprule
	case     & continue to $W$ & yield $W$ \\ \midrule 
  init $V \to W$   & $\falseInTable$         & $\falseInTable$    \\
  init $V \gets W$ & $W \not\in\bX\cup\bY\cup\bZii$   & $W\in\bZii$      \\
	$\to V\to W$      & $\falseInTable$         & $\falseInTable$         \\
	$\gets V\gets W$  & $W \not\in\bX\cup\bY\cup\bZii$   & $W\in\bZii$               \\
	$\gets V\to W$    & $\falseInTable $        &  $\falseInTable$              \\
	$\to V\gets W$    & $\falseInTable$         &  $\falseInTable$ \\ \bottomrule
\end{tabular}

\BlankLine


Let $\bZXY$ be given by Algorithm~\ref{alg:general:visit} from
vertices $\bX$ with rule table $\qquad$
\begin{tabular}{LLL} \toprule
	\text{case}     & \text{continue to } W                         &
  \text{yield }W \\ \midrule
  \text{init } V \to W   & W\not\in\bX\cup\bY\cup\bI\cup\bZ_{\text{An}}  &   W\in\bZ_{\text{An}}             \\
  \text{init } V \gets W & \falseInTable &  \falseInTable
  \\ 
	\to V\to W      & W\not\in\bX\cup\bY\cup\bI\cup\bZ_{\text{An}} &   W\in\bZ_{\text{An}}    \\
	\gets V\gets W  & \falseInTable &  \falseInTable     \\
	\gets V\to W    & \falseInTable &  \falseInTable     \\
	\to V\gets W    & \falseInTable &  \falseInTable     \\ \bottomrule
\end{tabular}

%


\BlankLine
Let $\bZZY$ be given by Algorithm~\ref{alg:general:visit} from vertices $\bI\cup\bZXY$ with rule table\vspace{1mm}
\begin{tabular}{LLL} \toprule
  \text{case}     & \text{continue to } W                         &
  \text{yield }W \\ \midrule
  \text{init}\ V \to W   & \falseInTable                                  & \falseInTable               \\
  \text{init}\ V \gets W & W \not\in \bX \cup \bI\cup\bZXY    & W\in\bZ_{\text{An}}                \\
  \to V\to W      & W \not\in \bX \wedge V \not\in \bI \cup \bZ_{\text{An}} & W\in\bZ_{\text{An}}    \wedge V \not\in \bI \cup \bZ_{\text{An}}  \\
  \gets V\gets W  & W \not\in \bX \cup \bI \cup\bZXY                               & W\in\bZ_{\text{An}}      \\
  \gets V\to W    &  W \not\in \bX \wedge V \not\in \bI \cup \bZ_{\text{An}} & W\in\bZ_{\text{An}}    \wedge V \not\in \bI \cup \bZ_{\text{An}}  \\
  \to V\gets W    & V  \in \bI \cup \bZ_{\text{An}} \wedge W \not\in \bX \cup
  \bI\cup\bZXY     & V  \in \bI \cup \bZ_{\text{An}} \wedge W\in\bZ_{\text{An}}      \\
  \bottomrule
\end{tabular}
\BlankLine

\Return{$\bI\cup\bZXY\cup\bZZY$}

\BlankLine

\end{algorithm*}


 \begin{proposition}\label{prop:minimal:algo:equivalence}
 Algorithm~\ref{alg:finding-minimal} and Algorithm~\ref{alg:finding:minimal:tabled} are equivalent.
 \end{proposition}
 \begin{proof}
 Both algorithms follow the same approach, computing sets $\bZii$, $\bZ_{\text{An}}$, $\bZXY$, $\bZZY$, and returning $\bI\cup \bZXY \cup \bZZY$. We need to show that each set is equal between the algorithms.
 
 $\bZii$ is equal because it is computed by algorithm Algorithm~\ref{alg:finding}.
 
 \def\bZaP{\bZ^{\ref{alg:finding-minimal}}_{\text{An}}}
 \def\bZaT{\bZ^{\ref{alg:finding:minimal:tabled}}_{\text{An}}}
In the following, we store the number of the algorithm
(\ref{alg:finding-minimal} and \ref{alg:finding:minimal:tabled}) as
superscript to distinguish between the sets and then show their equality.
 Both $\bZaP$ and $\bZaT$ include $\Pa(\bY) \cap \bZii$. $\bZaP$ as
 "parent of $\bY$", $\bZaT$ from rule $\mathrm{init}\ V \gets W$. 
 If $\bZaP$ includes a non-parent~$V$, there exists a directed path from $V$ to $\bY$ 
 containing no vertex of $\bX\cup (\bZii\setminus \{V\})$. The shortest
 such path contains only one vertex of $\bY$. Starting from that
 vertex,
 Algorithm~\ref{alg:finding:minimal:tabled}  will take rule
 $\mathrm{init}\ V\gets W$ and then rule $\gets V\gets W$, until it reaches $V$ of $\bZaP$, which is yielded since it is in $\bZii$.  
 In the other direction, if $V \in \bZaT$, there exists a path from $\bY$ to $V$ containing only edges $\gets$ and not containing any vertex in $\bX\cup\bY\cup\bZii$ except at the endvertices due to Lemma~\ref{lem:table:path}, so $V\in\bZaP$.

 If $\bZXYP$ contains a vertex $Z$, there exists a directed proper path from $\bX$ to $Z$ 
 containing no vertex of $\bI\cup(\bZ_{\text{An}}\setminus \{Z\})$.
 It also contains no vertex of $\bY$, otherwise $\bZii$ would not be an FD set since there is an unblocked path from $\bX$ to $\bY$.
 Algorithm~\ref{alg:finding:minimal:tabled} visits the vertices along
 that path since the edges are $\to$ and no non-end vertex is in $\bX\cup\bY\cup\bI\cup\bZ_{\text{An}}$. $Z$ is in $\bZ_{\text{An}}$, so it is yielded into $\bZXYT$. 
 In the other direction, if $V \in \bZXYT$, there exists a path from $X\in\bX$ to $Z$ not containing a vertex of $\bX\cup\bY\cup\bI\cup\bZ_{\text{An}}$ except $X$ and $Z$ at the end.

 If $\bZZYP$ contains a vertex $Z$,
 there exists a proper BD way from $\bI\cup\bZXY$ to $Z$ containing no edge $Z_{\text{An}}\to$ with $Z_{\text{An}}\in\bI\cup\bZ_{\text{An}}$ and no vertex of $\bX$, and all colliders are in $\bI\cup\bZ_{\text{An}}$.  
 
 Algorithm~\ref{alg:finding:minimal:tabled} starts at $\bI\cup\bZXY$ visiting the vertices along the path. If the path only contains one edge, $Z$ is yielded in the initial rule since it is in $\bZ_{\text{An}}$. 
 In the case $\gets V\gets W$ and $\gets V\to W$, $V$ is not in $\bI \cup \bZ_{\text{An}}$, or the path would contain $V\to$. $W$ is not in $\bX$ since the path contains no vertex of $\bX$.
 In the case $\gets V\gets W$, $W$ is not in $\bX$ as well, and $W$ is not in $\bI \cup\bZXY $ since the path is proper.
 In the case $\to V\gets W$, $W$ is not in $\bX\cup\bI \cup\bZXY $ for the same reasons. $V$ is in $\bI\cup\bZ_{\text{An}}$ since it is a collider.
 The final vertex $Z$ is yielded, once it is reached, since it is in $\bZ_{\text{An}}$. 
 
 For $Z\in\bZZYT$,  Algorithm~\ref{alg:finding:minimal:tabled} finds a
 way $\pi$ from $\bI\cup\bZXY$ to $Z\in\bZ_{\text{An}}$. It is a BD way since it starts with $\gets$ from the initial rule.
 If it contains a causal edge $V\to$, it is either the case/rule $\to V\to W$ or $\gets V\to W$, and in either case $V\not\in \bI\cup\bZ_{\text{An}}$. It contains no vertex of $\bX$ since $W\not\in\bX$ in all five rules. The only collider occurs in rule $\to V\gets W $ and it is in $\bI \cup \bZ_{\text{An}}$.\\
 If $\pi$ is not proper, there exists a subway $\pi'$ that is proper.
 Suppose that subway is not a BD path. Then it starts with  $V \to $ where $V\in\bI\cup\bZXY\subseteq \bI\cup\bZ_{\text{An}}$.  But $\pi$ and $\pi'$ cannot contain this $V \to $ because rules $  \to V\to W $ and $  \gets V\to W $ only continue or yield, if $V\not\in\bI \cup \bZ_{\text{An}}  $.
 %
 %
 \end{proof}

 We are now able to give the proof of
 Theorem~\ref{thm:finding:minimal}.
%

 \begin{proof}[Proof of Theorem~\ref{thm:finding:minimal}]
First we show that the algorithm returns a set that is an FD set if and only if an FD set exists.
If $\bZii$ can be computed, an FD set exists, by correctness of
Algorithm~\ref{alg:finding}. Moreover, $\bZ$ satisfies condition FD(1)
because $\bZ_{\text{An}}$ and $\bI\cup\bZXY$ block all directed paths
from $\bX$ to $\bY$ and $\bZ$ satisfies condition FD(2) because it is a subset of $\bZii$.

Suppose FD(3) is not satisfied. There is a proper BD path $\pi$
from  $\bZ$ to $\bY$ open given $\bX$ in $G_{\underline{\bZ}}$. If
$\pi$ would contain a collider, the collider is opened by $\bX$, so
there exists a path from $\bZ$ to the collider to $\bX$, and $\bZ$
violates condition FD(2). Likewise, $\pi$ cannot contain a vertex of
$\bX$. Furthermore, $\pi$ does not contain $\bI\to$ or it would not be
open. As $\pi$ is not open in $G_{\underline{\bZii}}$, there is an edge $Z \to$ for a vertex $Z\in \bZ_{(ii)}\setminus\bZ$ on $\pi$. The edge is directed as $Z \to$ because 
$Z \gets$ would start another BD path.
Since $\pi$ does not contain a collider, all following edges are directed towards $ \bY$ and $Z$ is an ancestor of $\bY$. 
Hence, $\pi$ contains an edge $Z_{\text{An}} \to $ with $Z_{\text{An}} \in \bZ_{\text{An}}$. We will choose the first such edge. 

There exists a BD way $\pi'$ from $\bI\cup\bZXY$ to $Z_{\text{An}}$: If
$\pi$ starts in $\bI\cup\bZXY$, $\pi'$ is the subpath of $\pi$
until $Z_{\text{An}}$. If $\pi$ starts in $\bZZY \in \bZXY$, then
there is a longer BD way  $\pi''$ from $\bI\cup\bZXY$ to $\bZZY$ found in step 4, and $\pi'$ is the concatenation of $\pi''$ and the subpath of $\pi$ until $Z_{\text{An}}$. 
Neither $\pi''$ nor $\pi$ contain a vertex of $\bX$ or a collider not in $\bI\cup\bZ_{\text{An}}$, so this is also true for $\pi'$. $\pi''$ does not contain an edge $\bI\cup\bZ_{\text{An}} \to$ and $Z_{\text{An}}\to$ is the first such edge on $\pi$, so $\pi'$ does not contain an edge $\bI\cup\bZ_{\text{An}}\to$. Hence $Z_{\text{An}} \in \bZXY$, and $\pi$ is blocked in $G_{\underline{\bZ}}$.

Now we show that $\bZ$ is minimal according to
Lemma~\ref{lem:characterization:minimal}.
Condition 1. and 2. are satisfied by the definition of $\bZXY$
and $\bZXY$. For the latter, the BD path to $V$ can be extended
to a BD path to $\bY$ because $V\in\bZ_{\text{An}}$ is an ancestor of $\bY$. All colliders in $\bI\cup\bZ_{\text{An}}$ are in $\bZXY$ because there is such a  way to them.

The maximal sets can be constructed in linear time by traversing the paths that need to be blocked. 
The first usable vertex of the path is added to the corresponding set. 
Thereby for the each constructed set, each edge is only visited at
most once and each vertex is added at most once to the set.
Algorithms~\ref{alg:general:visit}
and~\ref{alg:finding:minimal:tabled} describe this graph traversal in
more detail and formally the linear run time follows from
Proposition~\ref{prop:minimal:algo:equivalence}. 
\end{proof}

\section{Run-Time Comparison for Real-Life DAGs}\label{appendix:further:runtimes}
We compare the performance of our algorithm with \textsc{jtb} one on real-world instances to further validate our empirical
claims. We report the run-times in seconds for the graphs in the \texttt{bnlearn}
repository in Table~\ref{table:bnlearn}. The stated times are averages over 10
runs. In each, variables $X$ and $Y$ were randomly chosen, moreover
half of the remaining variables, chosen again randomly, were included
in $\bR$. 

The experiments show that \textsc{jtb} is significantly slower
compared to our algorithms \textsc{find} and \textsc{min} on
\emph{all} instances. For larger graphs, this results
in run-times above one minute for \textsc{jtb}, while the our
algorithms take fractions of milliseconds. 
In practice, this difference matters, especially when the
procedure is called many times (e.g. searching for front-door sets in
multiple graphs or for many pairs of variables).  
\begin{table}
  \caption{Run-time in seconds of Python implementations of \textsc{find} and
    \textsc{min} (Algorithms~2 and~3 from the main paper) compared to
  \textsc{jtb} (Jeong et al.) on the \texttt{bnlearn} instances.
  In the table, $n$ denotes the number of variables and $m$ the number
  of edges of the instance. 
}
  \label{table:bnlearn}
  \centering
  \begin{tabular}{lccrrr}
    \toprule
    \multicolumn{3}{c}{Instance} & \multicolumn{3}{c}{Run-time in
    seconds} \\
    \cmidrule(r){1-3} \cmidrule(l){4-6}
      Name & $n$ & $m$ & \textsc{find} & \textsc{min} & \textsc{jtb} \\
    \midrule
      asia & 8 & 8 & 0.00004 & 0.00005 & 0.00187 \\
      cancer & 5 & 4 & 0.00003 & 0.00004 & 0.00098 \\ 
      earthquake & 5 & 4 & 0.00003 & 0.00005 & 0.00157 \\ 
      sachs & 11 & 17 & 0.00003 & 0.00005 & 0.00667 \\ 
      survey & 6 & 6 & 0.00004 & 0.00005 & 0.00131 \\ 
      alarm & 37 & 46 & 0.00005 & 0.00010 & 0.10388 \\
      barley & 48 & 84& 0.00009 & 0.00013 & 0.14188 \\ 
      child & 20 & 25 & 0.00004 & 0.00007 & 0.01717 \\ 
      insurance & 27 & 52 & 0.00006 & 0.00010 & 0.04396 \\ 
      mildew & 35 & 46 & 0.00007 & 0.00011 & 0.04161 \\ 
      water & 32 & 66 & 0.00005  & 0.00011 & 0.07481 \\ 
      hailfinder & 56 & 66 & 0.00008 & 0.00011 & 0.18193 \\ 
      hepar2 & 70 & 123 & 0.00006 & 0.00011 & 0.64581 \\ 
      win95pts & 76 & 112 & 0.00006 & 0.00009 & 0.35280 \\ 
      andes & 223 & 338 & 0.00029 & 0.00042 & 11.27554 \\ 
      diabetes & 413 & 602 & 0.00074 & 0.00107 & 42.71703 \\ 
      link & 724 & 1125 & 0.00030 & 0.00040 & 194.19350 \\ 
      munin1 & 186 & 273 & 0.00011 & 0.00016 & 3.41474 \\ 
      pathfinder & 109 & 195 & 0.00005 & 0.00012 & 1.75494 \\ 
      pigs & 441 & 592 & 0.00014 & 0.00019 & 10.54953 \\ 
      munin & 1041 & 1397 & 0.00028 & 0.00043 & 145.71717 \\ 
      munin2 & 1003 & 1244 & 0.00028 & 0.00039 & 128.43397 \\ 
      munin3 & 1041 & 1306 & 0.00030 & 0.00044 & 236.72920 \\ 
      munin4 & 1038 & 1388 & 0.00030 & 0.00044 & 179.49700 \\ \bottomrule
  \end{tabular}
\end{table}

 \section{Further Experimental Results}\label{appendix:further:experiments}
In the experimental setup of this section, we consider Erd\H{o}s-R\'enyi random graphs with
$1.5n$, $2.5n$ and $5n$ edges, which correspond to expected degree
(i.e., number of neighbors) $3$, $5$ and $10$. For the number of
vertices, we choose powers of $2$ (as this gives us a wide range of
graph sizes), namely $2^4$ up to $2^{17}$, which corresponds to 131072
vertices. Furthermore, we consider
two settings for choosing the size of $\bX$ and $\bY$: (i) we choose
both sizes randomly between $1$ and $3$ and (ii) we let $\bX$ and
$\bY$ grow logarithmically in the number of vertices (more precisely,
we have $|\bX| = |\bY| = \log_2(n)-3$; we subtract 3 such that $|\bX|
= |\bY| = 1$ for our starting case of $2^4$ vertices).
Finally, we consider two sizes of $|\bR|$: (a) $|\bR| = 0.5n$ and (b)
$|\bR| = 0.25n$. 

Our experimental setup is limited in that we only consider the case
$\bI = \emptyset$ (which is the most natural setting; moreover, at least
for $\textsc{find}$ (Algorithm~\ref{alg:finding}) and \textsc{jtb}
(\cite{jeong2022finding}), this set plays no significant
algorithmic role) and that we only consider Erd\H{o}s-R\'enyi random graphs
(and not other classes of random graphs). However, the setting above
already yields $14 \cdot 3 \cdot 2 \cdot 2 = 168$ different parameter
choices and every further junction at least doubles it. 

We aim to corroborate two empirical claims in this section:
\begin{itemize}
  \item Our algorithms \textsc{find} and \textsc{min} can be
    implemented efficiently and scale to graphs with a large number of variables.
  \item The size of the maximal FD set (returned by \textsc{find} and
    \textsc{jtb}) is often significantly larger compared to the
  minimal FD set returned by (\textsc{min}).
\end{itemize}

As we will see, the run time of \textsc{find} and \textsc{min} is quite robust to
different choices for $|\bX|$, $|\bY|$ and $|\bR|$. While the run time
naturally increases with larger $n$ and $m$, it stays well-below a
second even for the largest instances considered here.

Similar observations can be made with regard to the maximal and
minimal FD sets. In particular, the maximal FD set is usually
extremely large, way larger than necessary to intercept the causal
paths from $\bX$ to $\bY$, and this also holds for all parameter
choices considered in this work.

A few technical details, before the results are discussed. We ran the
experiments on a single core of the AMD
Ryzen Threadripper 3970X 32-core processor on a 256GB
RAM machine. The final run of the experiments took about four days,
but trial runs were done before that (overall computation time was
around one to two weeks). The experiments, however, do not
need the large RAM provided and can be performed on average desktop
computers in a similar timeframe.

We implement our algorithms \textsc{find} and \textsc{min} in Python,
Julia, and JavaScript and
make our code publically available. 
For more clarity, we show only a selection of results on the following
pages, but will offer the remaining plots
online. They can be accessed at
\url{https://github.com/mwien/frontdoor-adjustment}. 

\begin{figure}
  \begin{center}
    \includegraphics[width=\textwidth]{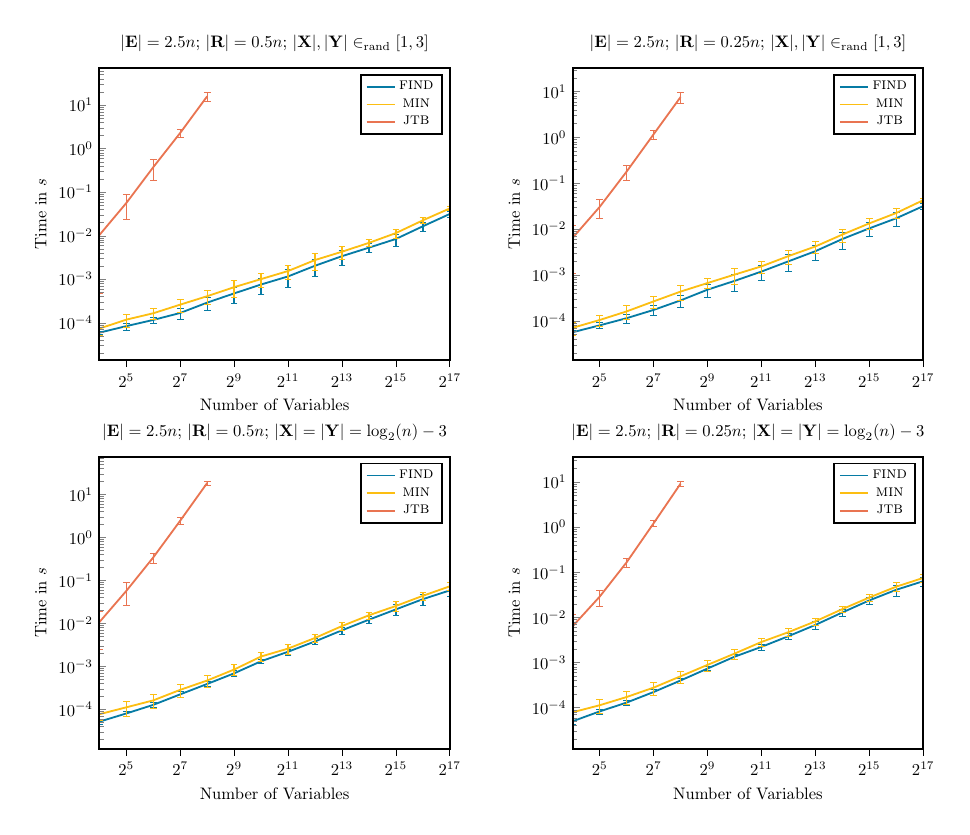}
  \end{center}
  \caption{Run time comparison between \textsc{find}
    (Algorithm~\ref{alg:finding}), \textsc{min}
    (Algorithm~\ref{alg:finding-minimal}) and
  \textsc{jtb} (\cite{jeong2022finding}) for additional parameter
  choices. In
particular, we choose $|\bE| = 2.5n$ (corresponding to expected degree
$5$) and vary the choices of $|\bR|$ and $|\bX|$, $|\bY|$. As can
be seen the run time differences are minor.}
\label{fig:exp:morepython}
\end{figure}

 \begin{figure}
   \begin{center}
     \includegraphics[width=\textwidth]{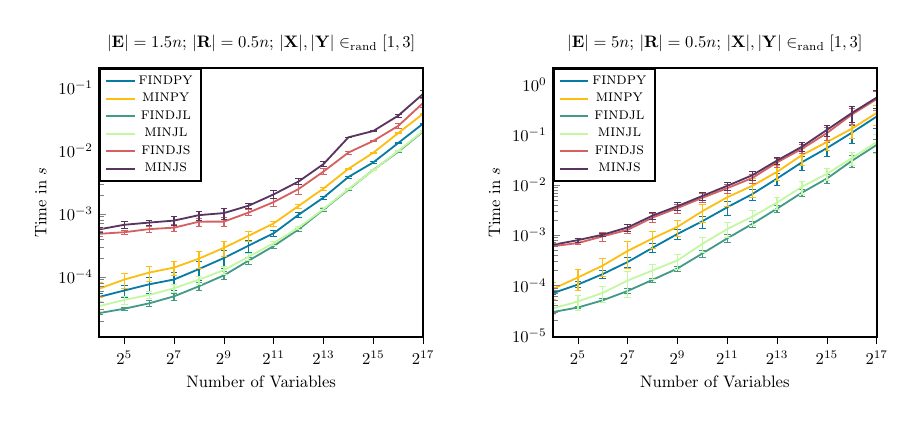}
   \end{center}
  \caption{Run time comparison of implementations of our algorithms
    (\textsc{find} and \textsc{min}) in the programming languages
    Python (\textsc{py}), Julia (\textsc{jl}), and JavaScript
    (\textsc{js}). The parameter choices are the same as
  in Fig.~\ref{fig:experiments} in the main paper. Generally, the Julia implementation is
fastest and Python is second fastest. The differences, however, are
comparably small (usually less than one order of magnitude).} 
\label{fig:exp:crosslang}
\end{figure}

Figure~\ref{fig:exp:morepython} shows additional run time comparisons
between \textsc{find}, \textsc{min}, and \textsc{jtb}. In contrast to
Fig.~\ref{fig:experiments} in the main paper, the number of edges is
chosen as $|\bE| = 2.5n$ and the choices of $|\bX|, |\bY|$, and $|\bR|$
are varied. 

The overall results differ only slightly. Unsurprisingly, the run time
of the algorithms mainly depends on the size of the graph (i.e., $n$
and $m$). For 256 variables, $\textsc{jtb}$ needs roughly 5-20
seconds per instance being barely within the time limit of 30 seconds.
Algorithms \textsc{find} and \textsc{min} only take fractions of a
second, even for the largest considered graphs with over a hundred
thousand vertices.

Figure~\ref{fig:exp:crosslang} shows a comparison over our different
implementations in Python, Julia, and JavaScript for the same parameter
choices as in Fig.~\ref{fig:experiments} in the main paper. The
differences between the languages are usually within one order of
magnitude with Julia being the fastest of the three. Our
implementations are \emph{not} specifically optimized and it would likely be
possible to improve the run time further (e.g. we store all sets
handled during the course of the algorithms in hashtables, raw arrays
could provide further performance benefits). 

Summing up, these experiments demonstrate that our implementations are
extremely efficient and able to handle all (currently) imaginable
graph instances arising in causal inference (and likely even more than
that). In the following, we analyse the properties of the randomly
generated instances in more detail. In particular, with regard to the size of the
found FD sets and the ratio of instances, which can be identified by
FD (and BD) adjustment.

\begin{figure}
  \begin{center}
    \includegraphics[width=\textwidth]{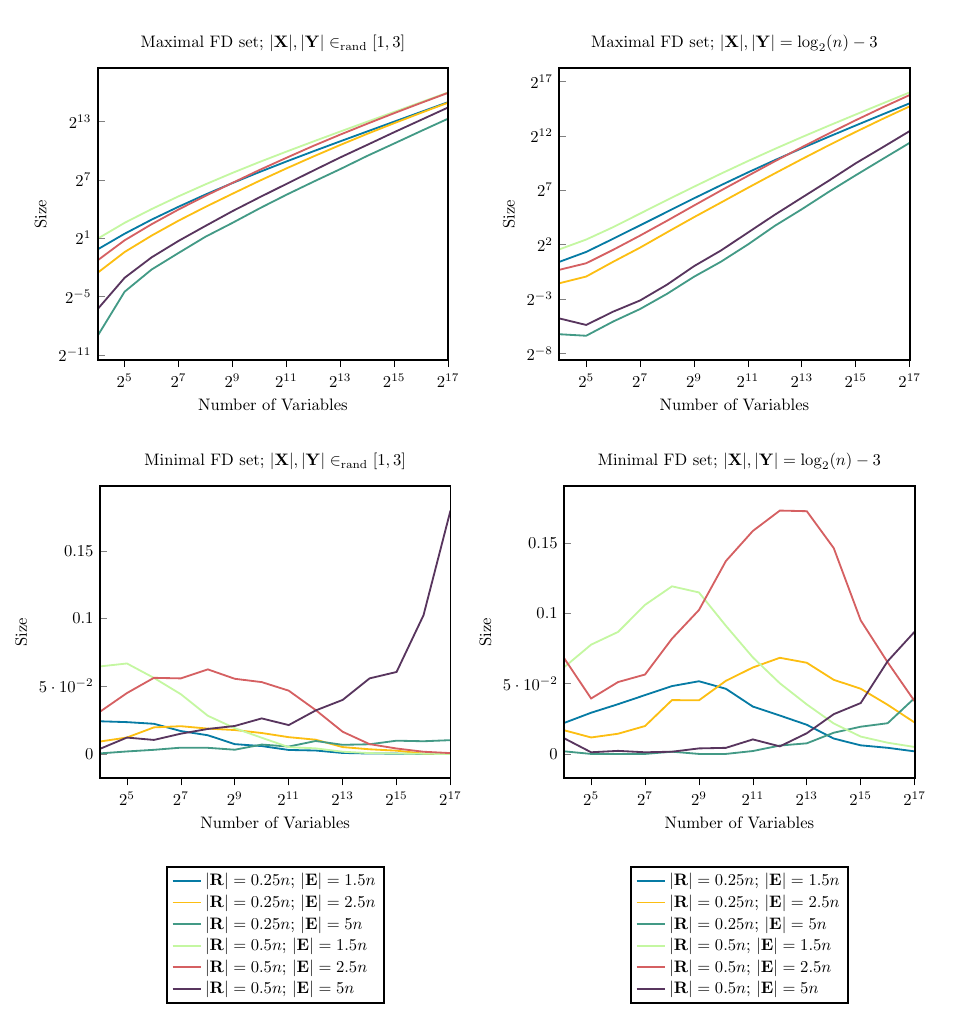}
  \end{center}
  \caption{Average size of the maximal FD set found by \textsc{find}
    and \textsc{jtb} (top) as well as the minimal FD set returned by
    $\textsc{min}$ (bottom). On the left instances for $|\bX|, |\bY|$
    chosen randomly between $1$ and $3$ are shown, on the right those
    with $|\bX| = |\bY| = \log_2(n) -3$. For the case of maximal FD sets,
    log-log plots are necessary to adequately show the results,
    whereas the average size of the minimal FD sets is extremely close
  to zero for many of the parameter choices. The averages are obtained
as the mean of the FD set sizes for all identified instances from
$10^4$ repetitions. Error bars are omitted for the sake of
readability, but standard deviations are reported in the online
supplement. 
}
\label{fig:exp:maxmin}
\end{figure}

Figure~\ref{fig:exp:maxmin} shows the average sizes of the maximal FD
sets, which are returned by \textsc{find} and \textsc{min}, and
the minimal FD set returned by \textsc{min} for $10^4$ randomly
generated instances per parameter choice (results for not FD identifiable instances are discarded). 

The maximal FD set is usually extremely large, growing linearly with
the number of vertices. This is not surprising as this set not
only includes the vertices necessary to block the causal paths between
$\bX$ and $\bY$, but all vertices which can be used for the FD set
(e.g., in the case of an empty graph, this set would still contain all
vertices in $\bR$).

For the same instances, \textsc{min} gives extremely small FD sets, many
of them actually having size zero, which corresponds to the case that
there is no causal path from $\bX$ to $\bY$ (there is bias introduced
to the experiments, by the fact that graphs without such a causal path are always
identified by FD, whereas those \emph{with} such a path are often not).
Generally, in the random graph model, it is hard to generate instance
with large minimal FD sets and this is an interesting topic for future
work. In the context of this work, the result demonstrates the stark contrast between
\textsc{find}, respectively \textsc{jtb}, and \textsc{min}, as the
former algorithm clearly produce unnecessarily large FD sets. 

\begin{figure}
    \begin{center}
      \includegraphics[width=0.9\textwidth]{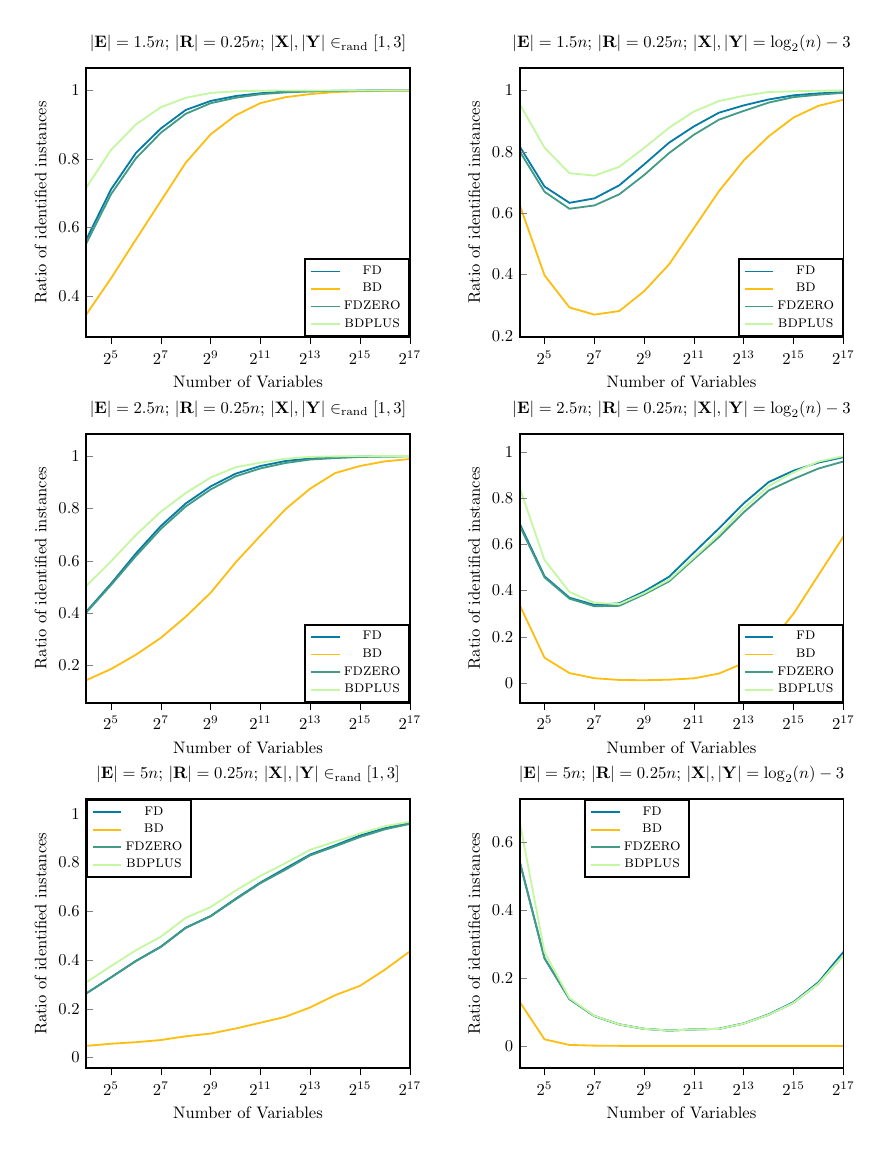}
  \end{center}
  \caption{Ratio of identified instances for $|\bR| = 0.25n$ and the
  same parameter choices as in Fig.~\ref{fig:exp:maxmin}. For
identification, we consider the FD criterion (\textsc{fd}), the
adjustment criterion~\citep{ShpitserVR2010,van2014constructing}
(\textsc{bd}), the FD criterion for $\bZ = \emptyset$
(\textsc{fdzero}) and \textsc{bd} combined with \textsc{fdzero}
(\textsc{bdplus}). The set up (with the addition of FD) is similar to the results presented in
Table 6 of~\cite{van2019separators}. }
  \label{fig:exp:ratios}
  \end{figure}

  Finally, in Fig.~\ref{fig:exp:ratios}, we report the ratio of
  identified instances for the case of $|\bR| = 0.25n$ and the settings
  in Fig.~\ref{fig:exp:maxmin} using (i) the FD criterion, (ii) the
  (BD) adjustment criterion~\citep{ShpitserVR2010,van2014constructing}, here
  denoted by \textsc{bd}, (iii) the FD criterion only allowing FD sets of size
  zero (\textsc{fdzero}) and (iv) the (BD) adjustment criterion
  together with \textsc{fdzero}, which we denote by \textsc{bdplus}.
  This setup (with the addition of the FD criterion) is similar to the
  one studied in~\citep{van2019separators} (Table 6). 

  As indicated already in Fig.~\ref{fig:exp:maxmin}, for Erd\H{o}s-R\'enyi
  random graphs, the difference between \textsc{fd} and
  \textsc{fdzero} is quite small. However, the results nicely show the
  orthogonal nature of FD and BD, as can be seen for example in the plot
  on the top right ($|\bE| = 2.5$ and
  $|\bX|,|\bY|=\log_2(n)-3$), where BD and FD alone 
  identify significantly less instances than \textsc{bdplus}.

\end{document}